\documentclass[11pt]{article}

\usepackage{acl}

\usepackage{times}
\usepackage{latexsym}
\usepackage{amsmath}
\usepackage{amssymb}
\usepackage{amsthm}
\usepackage{amsfonts}
\usepackage{float}
\usepackage{subcaption}
\usepackage{algorithm}
\usepackage{algpseudocode}

\usepackage[T1]{fontenc}

\usepackage[utf8]{inputenc}

\usepackage{microtype}

\usepackage{inconsolata}

\usepackage{graphicx}
\usepackage{booktabs}
\usepackage{colortbl}
\usepackage{enumitem}
\usepackage{bm}

\newtheorem{theorem}{Theorem}
\newtheorem{lemma}{Lemma}
\newtheorem{proposition}{Proposition}
\newtheorem{corollary}{Corollary}
\newtheorem{assumption}{Assumption}
\theoremstyle{remark}
\newtheorem{remark}{Remark}
\theoremstyle{plain}

\definecolor{tableheader}{RGB}{245,247,250}
\definecolor{llamablue}{RGB}{223,236,247}
\definecolor{optgreen}{RGB}{224,242,226}
\definecolor{grzopink}{RGB}{252,237,241}
\definecolor{grzopinkdeep}{RGB}{250,230,236}
\definecolor{foorange}{RGB}{224,118,0}
\definecolor{deltapos}{RGB}{0,128,0}
\definecolor{deltaneg}{RGB}{178,34,34}
\newcommand{\dpos}[1]{\textcolor{deltapos}{#1}}
\newcommand{\dneg}[1]{\textcolor{deltaneg}{#1}}

\DeclareMathAlphabet\mathbfcal{OMS}{cmsy}{b}{n}
\newcommand{\mat}[1]{\mathbf{#1}}

\newcommand{\forow}{\textcolor{foorange}{\textbullet}\,}

\title{GRZO: Group-Relative Zeroth-Order Optimization \\ for Large Language Model Fine-Tuning}

\author{
  \textbf{Liyan Tan} \quad
  \textbf{Yequan Zhao} \quad
  \textbf{Yifan Yang} \quad
  \textbf{Ruijie Zhang} \quad
  \textbf{Xinling Yu} \quad
  \textbf{Zheng Zhang} \\
  University of California, Santa Barbara \\
  \{liyan\_tan, yequan\_zhao, yifanyang, ruijiezhang, xyu644\}@ucsb.edu \\
  zhengzhang@ece.ucsb.edu
}

\hypersetup{pdfauthor={Liyan Tan, Yequan Zhao, Yifan Yang, Ruijie Zhang, Xinling Yu, Zheng Zhang},
            pdftitle={GRZO: Group-Relative Zeroth-Order Optimization for Large Language Model Fine-Tuning},
            pdfsubject={Zeroth-order optimization for large language model fine-tuning},
            pdfkeywords={zeroth-order optimization, large language model fine-tuning, memory-efficient training, variance reduction}}

\begin{document}
\maketitle
\begin{abstract}
Zeroth-order (ZO) optimization is a memory-efficient alternative to backpropagation for fine-tuning large language models, but its deployment is limited by the high variance of gradient estimation. We propose \textbf{GRZO}, a \textbf{G}roup-\textbf{R}elative \textbf{Z}eroth-\textbf{O}rder optimizer that draws one pseudo-independent perturbation per mini-batch example and aggregates the per-example losses through group-relative normalization, raising the effective gradient-direction count from one to the batch size at no additional forward cost while preserving inference-level memory. We prove that GRZO is directionally unbiased with variance shrinking proportionally to the batch size, yielding a tighter nonconvex convergence bound than MeZO. Across RoBERTa-large, Llama3-8B, and OPT-13B over multiple tasks, GRZO improves average accuracy on Llama3-8B by $+3.0$ over MeZO while staying within $0.5\%$ of the forward-only inference memory floor; as a drop-in replacement for the MeZO core, it lifts sparse, low-rank, and quantized ZO variants by $+4.9$ on average.
\end{abstract}

\section{Introduction}

\begin{figure*}[!t]
\centering
\begin{subfigure}[b]{0.56\textwidth}
    \centering
    \includegraphics[width=\textwidth]{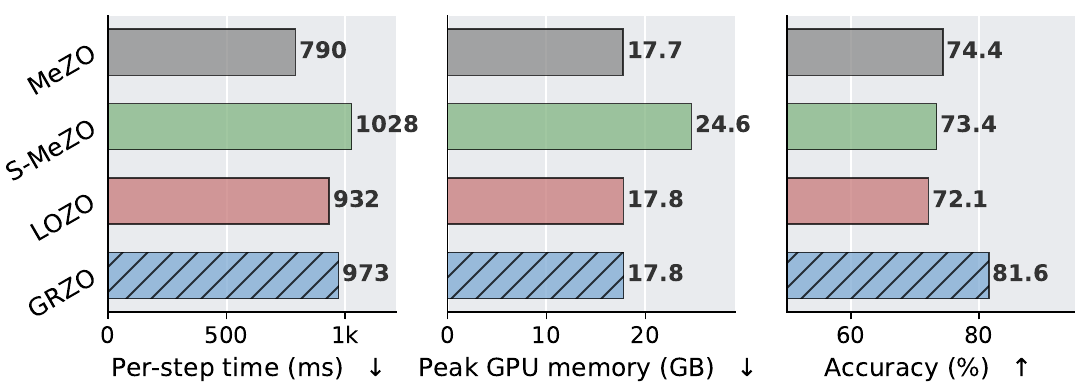}
    \caption{Efficiency and accuracy comparison.}
    \label{fig:intro_resource}
\end{subfigure}
\hfill
\begin{subfigure}[b]{0.42\textwidth}
    \centering
    \includegraphics[width=\textwidth]{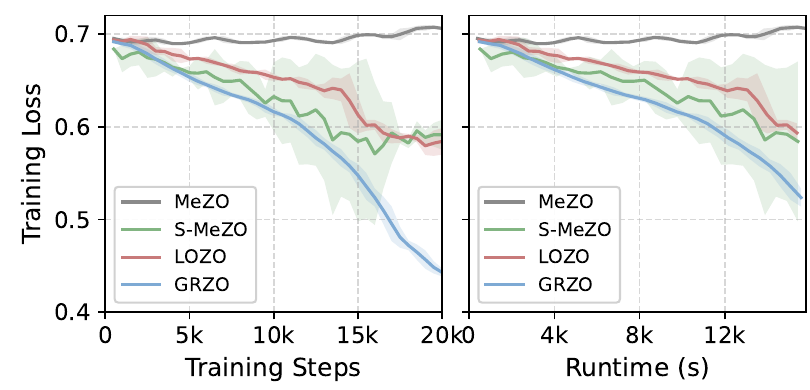}
    \caption{Loss convergence.}
    \label{fig:intro_convergence}
\end{subfigure}
\caption{GRZO at a glance on RTE (Llama3-8B). \textbf{Left}: highest accuracy ($81.6\%$) at inference-level peak memory ($17.82$~GB, $+0.5\%$ over the forward-only floor) for a $23\%$ per-step time premium over MeZO. \textbf{Right}: fastest convergence in both training steps and wall-clock time.}
\label{fig:intro_overview}
\end{figure*}

Fine-tuning large language models (LLMs) for downstream tasks remains essential, yet first-order fine-tuning is expensive: backpropagation requires storing activations, gradients, and optimizer states, and these costs grow linearly with model scale. Memory-efficient methods such as LoRA~\citep{hu2022lora}, Adapter~\citep{houlsby2019parameter}, Prefix-Tuning~\citep{li2021prefix}, Prompt-Tuning~\citep{lester2021power}, GaLore~\citep{zhao2024galorememoryefficientllmtraining}, CoLA~\citep{liu2025cola}, and LaX~\citep{zhang2026lax} reduce certain memory footprints but still rely on backpropagation, inheriting most of its activation-storage cost. Moreover, many practical objectives---accuracy, F1, reward signals---are non-differentiable and lie outside the first-order pipeline; the same barrier arises beyond language modeling, for instance in analog/RF circuit design, where the objective is returned by a simulator that exposes no usable gradient~\citep{tan2026zoaf}. These considerations motivate zeroth-order (ZO) fine-tuning as a forward-only alternative.

The canonical ZO method for LLM fine-tuning is MeZO~\citep{malladi2023fine}, a two-point estimator that approximates the gradient from the loss difference between two perturbed forward passes. \citet{malladi2023fine} report up to a $12\times$ memory reduction over SGD~\citep{amari1993backpropagation} and AdamW~\citep{loshchilov2017decoupled} fine-tuning, keeping training memory near inference levels while remaining compatible with non-differentiable objectives. The catch is that MeZO uses a single random perturbation direction per step; the variance of this estimator grows with model dimension, producing slow descent and brittle optimization once the backbone reaches the multi-billion-parameter range.

A growing literature addresses this variance by reducing the dimension via low-rank~\citep{chen2025enhancing} or sparse~\citep{liu2024sparse,zhang2025mazo} perturbations; in LLM fine-tuning, however, the loss landscape exhibits low effective rank~\citep{aghajanyan2021intrinsic, malladi2023kernel}, so the convergence rate can be independent of the parameter count. Another direction designs lower-variance ZO gradient estimators via control variates~\citep{gautam2024mezosvrg}, Hessian curvature~\citep{zhao2024hizoo}, or minimum-variance two-point estimators~\citep{ma2025revisiting}.

Both families buy variance reduction at a cost---a narrowed update space, extra forward passes, or extra persistent memory---eroding the inference-level efficiency that motivated ZO in the first place. We identify a third, long-overlooked axis for variance reduction: \emph{the mini-batch itself}. Existing ZO methods reuse a single perturbation direction across all $B$ examples of a step, even though the loss is evaluated per example. Drawing $B$ pseudo-independent directions instead---one per example---would, by standard Monte Carlo, reduce the SPSA (simultaneous perturbation stochastic approximation)~\citep{spall2002multivariate} estimator's variance by a factor of $1/B$ at no additional forward cost, no parameter-space restriction, and no extra persistent memory. Realizing this axis efficiently, while preserving MeZO's two-forward-pass budget and inference-level memory footprint, is the central design problem of this paper.

We propose GRZO (Group-Relative Zeroth-Order Optimization), which realizes this axis by drawing $B$ pseudo-independent perturbations via Flipout-style sign factorization~\citep{wen2018flipout} and aggregating the resulting per-example loss differences through GRPO-style group-relative normalization~\citep{shao2024deepseekmath}, all within a single two-forward-pass step. Figure~\ref{fig:intro_overview} summarizes the resulting accuracy, memory, and convergence profile on RTE. Our contributions are:
\begin{itemize}[leftmargin=*]
    \item \textbf{Algorithm.} A ZO optimizer that turns the mini-batch into pseudo-independent perturbation directions while preserving MeZO's two-forward-pass budget and inference-level memory.
    \item \textbf{Theory.} We show the directional unbiasedness and batch-size-scaled variance reduction of GRZO, yielding a strictly tighter nonconvex convergence bound than single-direction ZO.
    \item \textbf{Experimental results.} We show that GRZO outperforms MeZO across models and tasks, and its sparse, low-rank, and quantized variants on most tasks. We further show that GRZO is complementary to sparse, low-rank, and quantized ZO variants and that they can be combined to achieve further performance benefit.
\end{itemize}

\noindent A mechanism-by-mechanism comparison of representative ZO methods is in Table~\ref{tab:related_work} (Appendix~\ref{app:related_work_table}); these approaches are largely orthogonal to GRZO and compose with it. 

\begin{figure*}[t]
\centering
\includegraphics[width=\textwidth]{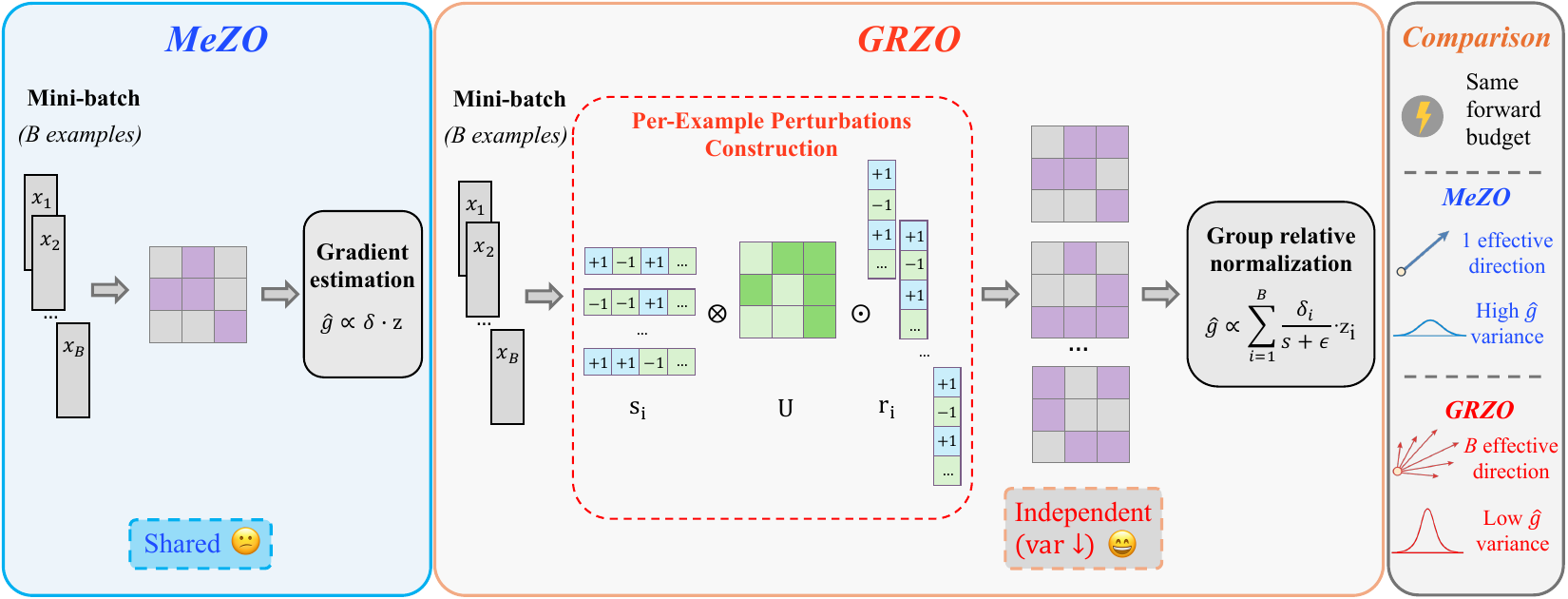}
\caption{Side-by-side pipeline comparison of MeZO (left) and GRZO (right). By constructing pseudo-independent perturbations and group-relative normalization, GRZO achieves $B$ effective perturbation directions and $1/B$ gradient variance under the same forward-pass budget compared to MeZO.}
\label{fig:grzo_overview}
\end{figure*}

\section{Background and Related Work}
\label{sec:related_work}

\subsection{Zeroth-Order (ZO) Optimization}

Zeroth-order (ZO) optimization~\citep{nesterov2017random, ghadimi2013stochastic} adjusts model parameters $\bm{\theta}\in\mathbb{R}^{D}$ using only forward queries of the loss $\mathcal{L}(\bm{\theta})$, avoiding the memory overhead caused by activation and gradient buffers in backpropagation. A ZO optimizer still performs gradient-descent update $\bm{\theta}_{t}\leftarrow\bm{\theta}_{t-1}-\alpha\,\mat{g}$, but it approximates the gradient $\mat{g}$ via $N$ forward passes
\begin{equation}
\mat{g} \approx \widehat{\nabla}_{\bm{\theta}}\mathcal{L}(\bm{\theta}) = \sum_{i=1}^N \frac{1}{N\mu} \left[\mathcal{L}(\bm{\theta}+\mu \bm{\xi}_i) - \mathcal{L}(\bm{\theta})\right] \bm{\xi}_i,
\label{eq:zo_gradient_estimation}
\end{equation}
with $\{\bm{\xi}_i\}_{i=1}^N$ drawn i.i.d.\ from an isotropic distribution $\rho(\bm{\xi})$ (e.g., $\mathcal{N}(\mat{0},\mat{I})$ or Rademacher) and $\mu>0$ a small sampling radius. The estimator $\widehat{\nabla}_{\bm{\theta}}\mathcal{L}$ is unbiased w.r.t.\ the gradient of the smoothed surrogate $f_\mu(\bm{\theta}):=\mathbb{E}_{\bm{\xi} \sim \rho}[\mathcal{L}(\bm{\theta}+\mu \bm{\xi})]$, but biased w.r.t.\ the true gradient $\nabla_{\bm{\theta}}\mathcal{L}$~\citep{berahas2022theoretical}, and its variance carries a dimension-dependent factor $O(D/N)$ at $\mu=O(1/\sqrt{N})$~\citep{liu2020primer, duchi2015optimal, gao2022generalizing}.

MeZO, one of the most popular ZO optimizers for LLM fine-tuning, is the $N{=}1$ two-sided instantiation of Eq.~\eqref{eq:zo_gradient_estimation}: a seed-regenerated direction $\bm{\xi}\in\mathbb{R}^{D}$ drives two symmetric forward passes $\ell^{\pm}=\mathcal{L}(\bm{\theta}\pm\mu\bm{\xi};\mathcal{B})$, and parameters are updated in place along $-\alpha(\ell^{+}-\ell^{-})\bm{\xi}/2\mu$ without materializing the perturbation tensor. Two knobs govern the estimator's quality: the parameter dimension $D$ over which its variance scales, and the SPSA construction itself. The ZO fine-tuning literature addresses these knobs separately.

\subsection{Reducing the Effective Dimension}
One family of approaches shrinks the update space to mitigate the $O(D)$ variance scaling. DeepZero~\citep{chen2024deepzero} and Sparse-MeZO~\citep{liu2024sparse} restrict updates to a sparsity mask, and MaZO~\citep{zhang2025mazo} extends masking to multi-task fine-tuning; 
Low-rank methods reparameterize the perturbation through low-rank matrices~\citep{chen2025enhancing} or tensors~\citep{zhao2023tensor, yang2024adazeta}.
These methods exploit the low effective rank of LLM fine-tuning at the cost of full-parameter expressivity. Along an orthogonal axis, QuZO~\citep{zhou2025quzo} and Poor-Man's Training~\citep{zhao2025poor} reduce memory via low-bit forward passes without altering the SPSA estimator.

\subsection{Improving the Estimator Construction}
A second family of research keeps the full-parameter update space and injects additional information into the estimator. MeZO-SVRG~\citep{gautam2024mezosvrg} pairs each probe with a periodic full-batch reference for an SVRG-style control variate, at the cost of doubled persistent memory; HiZOO~\citep{zhao2024hizoo} adds a diagonal-Hessian preconditioner estimated from one extra forward pass per step; FZOO~\citep{dang2025fzoo} samples $N$ parallel directions for $1/N$-scale variance reduction at $N{+}1$ forwards per step; subspace-orthogonalization~\citep{lang2026powering} decorrelates the direction sequence across steps; and \citet{ma2025revisiting} revisit minimum-variance two-point estimator design; SharpZO~\citep{yang2026sharpzo} extends the forward-only paradigm to sharpness-aware VLM prompt tuning. In every case, the variance reduction is achieved with extra forwards, extra persistent memory, or both.

\section{Method}
\label{sec:method}

Both ZO families reviewed in Section~\ref{sec:related_work} reduce variance at a cost: shrinking the effective parameter dimension sacrifices full-parameter expressivity, while enriching the SPSA estimator adds forward passes or persistent memory. GRZO instead turns the mini-batch dimension itself into the variance-reduction lever: a single two-forward-pass step yields $B$ pseudo-independent gradient directions, preserving the inference-level memory and two-forward-pass envelope of MeZO. Section~\ref{sec:grzo_flipout} constructs the per-example perturbation directions via a sign factorization of a shared base perturbation; Section~\ref{sec:grzo_core} then converts the resulting $B$ per-example loss signals into group-relative weights inspired by GRPO advantages. Figure~\ref{fig:grzo_overview} contrasts the resulting pipeline with MeZO.

\subsection{Per-Example Perturbations via Structured Injection}
\label{sec:grzo_flipout}
GRZO constructs per-example perturbations through a sign factorization originally proposed for Bayesian weight sampling~\citep{wen2018flipout}. For a linear layer with weight $\mat{W}\in\mathbb{R}^{d_{\text{out}}\times d_{\text{in}}}$, we first generate a shared base perturbation matrix $\mat{U} \in \mathbb{R}^{d_{\text{out}}\times d_{\text{in}}}$ in each step from a symmetric isotropic distribution (e.g., $\mathcal{N}(0,1)$ or Rademacher entries). Then for each data sample in $\{ \mat{x}_i, \mat{y}_i\}_{i=1}^B$, we modulate $\mat{U}$ with per-example sign vectors $\mat{r}_i\in\{\pm 1\}^{d_{\text{out}}}$ and $\mat{s}_i\in\{\pm 1\}^{d_{\text{in}}}$ (independent Rademacher pairs) to obtain
\begin{equation}
\Delta \mat{W}_i \;=\; \mat{U} \odot (\mat{r}_i \mat{s}_i^\top),
\label{eq:flipout_deltaW}
\end{equation}
where $\odot$ denotes the element-wise (Hadamard) product. This construction yields $B$ pseudo-independent per-example perturbations from a single shared base $\mat{U}$, without materializing $B$ separate weight copies; see Figure~\ref{fig:grzo_overview} for the pipeline and Appendix~\ref{app:background_details} for the vectorized form. The overhead is a small constant per linear layer (one extra matrix multiplication plus elementwise sign modulations).

\begin{lemma}[Isotropy and conditional decorrelation]
\label{lem:flipout_indep}
Let $\mat{z}_i := \mathrm{vec}(\Delta\mat{W}_i) = \mathrm{vec}(\mat{U} \odot (\mat{r}_i \mat{s}_i^\top))$ with $\mat{U}$ and $\{(\mat{r}_i, \mat{s}_i)\}_{i=1}^B$ as defined above. Then $\mathbb{E}[\mat{z}_i\mat{z}_i^\top]=\mat{I}_{d_{\text{out}}d_{\text{in}}}$ marginally and $\mathbb{E}[\mat{z}_i\mat{z}_j^\top\mid \mat{U}]=0$ for $i\neq j$. Drawing an independent base $\mat{U}^{(\ell)}$ and independent sign pairs per layer, the concatenation $\mat{z}_i=(\mat{z}_i^{(1)},\dots,\mat{z}_i^{(L)})\in\mathbb{R}^{D}$ over all $L$ perturbed blocks inherits both properties on $\mathbb{R}^{D}$, where $D$ is the total number of perturbed parameters (Appendix~\ref{app:flipout_decorrelation}).
\end{lemma}

\subsection{Group-Relative Aggregation for Zeroth-Order Updates}
\label{sec:grzo_core}
Given the per-example perturbations $\Delta\mat{W}_i$ in Eq.~\eqref{eq:flipout_deltaW} and a perturbation scale $\sigma > 0$, example $i$ contributes the two-sided perturbed losses $\ell_i^{\pm} = L(\boldsymbol{\theta} \pm \sigma\Delta\mat{W}_i; \mat{x}_i, \mat{y}_i)$ and the perturbation-induced loss difference 
\begin{equation}
  \delta_i := \ell_i^{+} - \ell_i^{-}. 
\end{equation}  
GRZO converts $\{\delta_i\}_{i=1}^{B}$ into advantage-like weights via group-relative normalization. We compute the within-batch standard deviation
\begin{equation}
s = \sqrt{\frac{1}{B}\sum_{i=1}^{B}\bigl(\delta_i - \bar\delta\bigr)^2}  \; \text{with}\;\bar\delta = \frac{1}{B}\sum_{i=1}^{B}\delta_i,
\label{eq:grzo_stats}
\end{equation}
and define the group-relative weights
\begin{equation}
a_i = \frac{\delta_i}{s+\epsilon}.
\label{eq:grzo_adv}
\end{equation}
Here $\epsilon>0$ is a small constant to ensure numerical stability. These weights are scale-invariant under rescaling of $\{\delta_i\}$, decoupling the update from loss-magnitude drift; two-sided differences also give $\mathbb{E}[\delta_i]=0$ by the $\mat{z}_i\leftrightarrow-\mat{z}_i$ symmetry, so the numerator needs no explicit mean-centering. The update direction is
\begin{equation}
\widehat{\mat{g}} \;=\; \frac{1}{2 \sigma B}\sum_{i=1}^{B} a_i\, \mat{z}_i,
\label{eq:grzo_estimator}
\end{equation}
where $\mat{z}_i$ is the per-example perturbation direction from the preceding Section~\ref{sec:grzo_flipout}; full pseudocode is in Algorithm~\ref{alg:grzo} (Appendix~\ref{app:grzo_algorithm}). The conditional decorrelation in Lemma~\ref{lem:flipout_indep} eliminates the cross-example covariance terms in $\mathrm{Var}(\hat{\mat{g}})$ \emph{given} the shared base $\mat{U}$; the residual fluctuation of $\mat{U}$ itself is not removed by per-example averaging and is accounted for separately in Theorem~\ref{thm:grzo_var} (Appendix~\ref{app:grzo_second_moment}).

\section{Theoretical Analysis}
\label{sec:theory}

The construction in Section~\ref{sec:method} promises $B$ pseudo-independent gradient directions per step at no extra forward cost. We now make this precise. Let $D$ denote the total number of perturbed parameters and let $\mat{z}_i\in\mathbb{R}^{D}$ be the concatenation of example $i$'s per-block Flipout perturbations (Lemma~\ref{lem:flipout_indep}). All statements below are for this \emph{joint} perturbation: a single loss difference $\delta_i$ probes every block at once, so the estimator restricted to one block still carries the perturbations and gradients of all other blocks and cannot be analyzed in isolation (Appendix~\ref{app:blockwise_full_network}). We first analyze the canonical estimator
\begin{equation}
\hat{\mat{g}}_{\mathrm{can}} \;=\; \tfrac{1}{2\sigma B}\textstyle\sum_{i=1}^{B}\delta_i\,\mat{z}_i,
\label{eq:g_can_main}
\end{equation}
and then quantify the finite-$B$ effect of the group-relative normalization that turns \eqref{eq:g_can_main} into the GRZO estimator \eqref{eq:grzo_estimator}. Proofs are in Appendix~\ref{app:grzo_unbiased}.

\begin{theorem}[Unbiasedness]
\label{thm:grzo_unbiased}
Under standard smoothness assumptions,
$\mathbb{E}[\hat{\mat{g}}_{\mathrm{can}} \mid \boldsymbol{\theta}] = \nabla F_\sigma(\boldsymbol{\theta}) + O(\sigma^2)$,
where $F_\sigma(\boldsymbol{\theta})=\mathbb{E}_{\mat{z}}[F(\boldsymbol{\theta}+\sigma \mat{z})]$ and the $O(\sigma^2)$ term is the standard ZO smoothing bias, exactly zero for Gaussian base noise.
\end{theorem}

\begin{theorem}[Variance Bound]
\label{thm:grzo_var}
Let $m_4=\mathbb{E}[\mat{U}_{jk}^4]$. Under the assumptions of Appendix~\ref{app:grzo_second_moment},
\small
\begin{align}
\mathrm{Var}\bigl(\hat{\mat{g}}_{\mathrm{can}}\bigr) &\le \frac{D-1}{B}\bigl(\|\nabla F\|^2+\nu^2\bigr) \nonumber\\
&\quad + (m_4-1)\Bigl(\|\nabla F\|^2+\frac{\nu^2}{B}\Bigr) + O(\rho^2\sigma^4 D^4). \nonumber
\end{align}
\normalsize
\end{theorem}

\paragraph{Which term shrinks with $B$.}
The first term is the per-example estimation noise: the $B$ conditionally decorrelated directions of Lemma~\ref{lem:flipout_indep} average it down by $1/B$. The second term is the fluctuation of the \emph{shared} base $\mat{U}$, which every example sees identically and which per-example averaging therefore cannot remove. For Rademacher base noise $m_4{=}1$ and this term vanishes identically, leaving the pure $1/B$ law; for Gaussian noise $m_4{=}3$ and it is a batch-independent floor. Its share of the bound is $O(B/D)$---below $10^{-8}$ for an 8B backbone at $B{=}16$---so it is numerically irrelevant at LLM scale, but it is a second reason, beyond the memory argument of Section~\ref{sec:ablation}, to take Rademacher as the default.

\begin{proposition}[Effect of group-relative normalization]
\label{prop:grzo_norm}
Let $\kappa_\delta=\mathbb{E}[\delta_i^4]/(\mathbb{E}[\delta_i^2])^2$, $s_\star^2=\mathbb{E}[s^2]$, $c_\star=1/(s_\star+\epsilon)$, and let $d_{\mathrm{eff}}=\|\nabla F\|^4/\sum_k(\partial_k F)^4$ be the participation ratio of the gradient. Then
\small
\begin{equation}
\frac{\mathrm{std}(s)}{s_\star}\;\le\;\sqrt{\frac{\kappa_\delta-1}{4B}}\;+\;O\bigl(d_{\mathrm{eff}}^{-1/2}\bigr),\nonumber
\end{equation}
\normalsize
and consequently, with $\Delta=O(B^{-1}+d_{\mathrm{eff}}^{-1})$,
$\mathbb{E}[\langle\nabla F_\sigma,\hat{\mat{g}}_{\mathrm{GRZO}}\rangle] = c_\star(1+\Delta)\|\nabla F_\sigma\|^2+O(\sigma^2)$ and
$\mathrm{Var}(\hat{\mat{g}}_{\mathrm{GRZO}})\le c_\star^2(1+\Delta)\,\mathrm{Var}(\hat{\mat{g}}_{\mathrm{can}})$.
\end{proposition}

The scalar $1/(s+\epsilon)$ is random and correlated with $\hat{\mat{g}}_{\mathrm{can}}$, so it cannot simply be absorbed into the learning rate. Proposition~\ref{prop:grzo_norm} states what it does instead: it concentrates around $c_\star$ at rate $O(B^{-1/2})$ and perturbs both the descent direction and the variance by a relative $O(1/B)$, changing constants but neither the $1/B$ scaling of Theorem~\ref{thm:grzo_var} nor the rate below. This correction grows as $B$ shrinks, which is the mechanism behind the small-batch instability of Appendix~\ref{app:bs-ablation}; the constant $\epsilon>0$ additionally caps $1/(s+\epsilon)\le 1/\epsilon$ and rules out tail blow-up.

\begin{assumption}
\label{ass:convergence}
$F_\sigma$ is $\mathcal{L}$-smooth and lower bounded by $F_\sigma^\star$; per-example gradients satisfy $\mathbb{E}\|\nabla\ell(\boldsymbol{\theta};\xi)-\nabla F(\boldsymbol{\theta})\|^2\le\nu^2$; and $\eta\le\min\{\tfrac{1}{2\mathcal{L}},\tfrac{1}{8\mathcal{L}\kappa}\}$ with $\kappa=\tfrac{D-1}{B}+m_4-1$.
\end{assumption}

\begin{theorem}[Nonconvex Convergence]
\label{thm:grzo_conv}
Under Assumption~\ref{ass:convergence}, GRZO iterations satisfy
\begin{align}
\frac{1}{T}\sum_{t=0}^{T-1}\mathbb{E}\big\|\nabla F_\sigma(\boldsymbol{\theta}_t)\big\|^2
&\;\le\;
\frac{8\bigl(F_\sigma(\boldsymbol{\theta}_0)-F_\sigma^\star\bigr)}{\eta T} \notag\\
&\quad + 4\mathcal{L}\eta\,\beta,
\label{eq:grzo_conv}
\end{align}
where $\beta=\frac{D+m_4-2}{B}\nu^2+2\kappa c_{\mathrm{bias}}^2\sigma^4D^2+O(\rho^2\sigma^4D^4)$.
(Proof: Theorem~\ref{thm:grzo_network_convergence}, Appendix~\ref{app:blockwise_full_network}.)
\end{theorem}

\paragraph{Comparison with MeZO.}
At the same batch size, MeZO draws a single direction $\mat{z}$ and evaluates the batch-mean loss, so its estimator is $\langle\bar{\mat{g}},\mat{z}\rangle\mat{z}$ with $\bar{\mat{g}}=\frac{1}{B}\sum_i\mat{g}_i$, giving $\mathrm{Var}(\hat{\mat{g}}_{\mathrm{MeZO}})\approx(D{+}m_4{-}2)(\|\nabla F\|^2+\nu^2/B)$. MeZO therefore already averages the \emph{data} noise over the mini-batch; what GRZO additionally averages is the \emph{directional} noise. Writing $r=\nu^2/\|\nabla F\|^2$, the ratio of the two bounds at the Rademacher default ($m_4{=}1$, where the shared-base term vanishes) is
\begin{equation}
\frac{\mathrm{Var}(\hat{\mat{g}}_{\mathrm{MeZO}})}{\mathrm{Var}(\hat{\mat{g}}_{\mathrm{can}})}\;\approx\;\frac{B+r}{1+r}\;=\;cB+(1-c)\;=:\;B_{\mathrm{eff}},
\label{eq:beff}
\end{equation}
where $c=1/(1+r)$ is the mean pairwise cosine between per-example gradients. This gives the informal ``$B$ effective directions'' of Section~\ref{sec:method} a measurable meaning: $B_{\mathrm{eff}}=B$ when per-example gradients are perfectly aligned ($c{=}1$) and degrades to $1$ when they are mutually orthogonal ($c{=}0$), with $c$ estimable directly from per-example gradients on the target model; Appendix~\ref{app:cosine} measures both $c$ and the resulting alignment with the exact backpropagation gradient. With $\eta\propto 1/\sqrt{T}$, \eqref{eq:grzo_conv} then improves MeZO's stationarity bound by $\sqrt{B_{\mathrm{eff}}}$: GRZO reaches an $\epsilon$-stationary point in $B_{\mathrm{eff}}\times$ fewer steps at exactly the same two forward passes per step, without extra perturbation queries. Proposition~\ref{prop:grzo_norm} additionally requires $B$ large enough for $s$ to concentrate, consistent with the $B\ge 16$ recommendation of Section~\ref{sec:grzo_bs_sensitivity}.

\section{Experiments}
\label{sec:experiments}

We consider two architectural families. Masked language models such as BERT~\citep{devlin2019bert} and RoBERTa~\citep{liu2019roberta} learn bidirectional representations under a masked-token objective; auto-regressive language models such as Llama~\citep{grattafiori2024llama} and OPT~\citep{zhang2022opt} predict the next token. We benchmark GRZO under full-parameter fine-tuning (Section~\ref{sec:lm_results}), report memory and per-step time on Llama3-8B (Section~\ref{sec:memory-analysis}), show GRZO composes with parameter-efficient ZO variants such as Sparse-MeZO, LOZO, and QuZO (Section~\ref{sec:composing}), and ablate the components of GRZO (Section~\ref{sec:ablation}).

\noindent \textbf{Setup.}
We compare against first-order baselines (Adam~\citep{kingma2014adam}, LoRA~\citep{hu2022lora}) and zeroth-order baselines (MeZO~\citep{malladi2023fine}, FZOO~\citep{dang2025fzoo}) on classification tasks from GLUE~\citep{wang2018glue} (SST-2) and SuperGLUE~\citep{wang2019superglue} (RTE, CB, BoolQ, WiC, MultiRC, COPA), and QA tasks SQuAD~\citep{rajpurkar2016squad} and DROP~\citep{dua2019drop}. On the autoregressive models all methods train for $20k$ steps at batch size 16 in FP16; ZO methods use perturbation scale $\sigma{=}10^{-3}$. For RoBERTa-large we follow the $k{=}512$ few-shot protocol of \citet{malladi2023fine}, including its batch size of 64. Hyperparameter details are in Appendix~\ref{app:exp_settings}.

\subsection{Accuracy and Convergence}
\label{sec:lm_results}

\begin{table}[t]
\centering
\small
\caption{Results on RoBERTa-large (350M, $k$=512). FO methods are marked with orange bullets. Among ZO methods, the highest accuracy is highlighted in \textbf{bold}.}
\label{tab:roberta-table}
\setlength{\tabcolsep}{3.5pt}
\renewcommand{\arraystretch}{1.12}
\begin{tabular}{@{}lcccccc@{}}
\toprule
\rowcolor{tableheader}
\textbf{Method} & \multicolumn{2}{c}{\textbf{Sentiment}} & \multicolumn{3}{c}{\textbf{NLI}} & \textbf{Topic} \\
\cmidrule(lr){2-3}\cmidrule(lr){4-6}\cmidrule(lr){7-7}
\rowcolor{tableheader}
 & \textbf{SST-2} & \textbf{SST-5} & \textbf{SNLI} & \textbf{MNLI} & \textbf{RTE} & \textbf{TREC} \\
\midrule
Zero-shot & 79.0 & 35.5 & 50.2 & 48.8 & 51.4 & 32.0 \\
\forow LP(FO) & 91.3 & 51.7 & 80.9 & 71.5 & 73.1 & 89.4 \\
\midrule
\forow FT(FO) & 91.9 & 47.5 & 77.5 & 70.0 & 66.4 & 85.0 \\
\midrule
MeZO & 92.8 & 53.2 & 83.0 & 78.3 & 78.6 & 94.3 \\
FZOO & 93.0 & 54.2 & 84.6 & \textbf{79.9} & 78.1 & \textbf{95.6} \\
\rowcolor{grzopink}
\textbf{GRZO} & \textbf{93.3} & \textbf{54.8} & \textbf{85.4} & 79.1 & \textbf{79.0} & 94.7 \\
\bottomrule
\end{tabular}
\end{table}

\paragraph{Masked Language Models.}
We evaluate GRZO on RoBERTa-large (350M) under the $k{=}512$ few-shot setting. Table~\ref{tab:roberta-table} shows GRZO outperforms MeZO on all six tasks and FZOO on four of six, with the largest gains over FZOO on RTE ($+0.9$) and SNLI ($+0.8$). Averaged across tasks, GRZO reaches 81.1, vs.\ 80.9 for FZOO and 80.0 for MeZO.

\begin{table*}[!t]
\centering
\small
\caption{Results on Llama3-8B across SuperGLUE and QA tasks. FO methods are marked with orange bullets. Among ZO methods, the highest accuracy is highlighted in \textbf{bold}.}
\label{tab:main-llm-table}
\setlength{\tabcolsep}{5.4pt}
\renewcommand{\arraystretch}{1.12}
\begin{tabular}{@{}lccccccccc@{}}
\toprule
\rowcolor{tableheader}
\textbf{Method} & \multicolumn{7}{c}{\textbf{SuperGLUE (Classification)}} & \multicolumn{2}{c}{\textbf{QA}} \\
\cmidrule(lr){2-8}\cmidrule(l){9-10}
\rowcolor{tableheader}
 & \textbf{SST-2} & \textbf{RTE} & \textbf{CB} & \textbf{BoolQ} & \textbf{WiC} & \textbf{MultiRC} & \textbf{COPA} & \textbf{SQuAD} & \textbf{DROP} \\
\midrule
\rowcolor{llamablue}
\multicolumn{10}{c}{\textbf{Llama3-8B}} \\
\forow Adam (FO) & 96.0 & 92.0 & 92.0 & 86.6 & 72.6 & 84.7 & 89 & 90.4 & 59.4 \\
\forow LoRA (FO) & 95.0 & 80.9 & 73.2 & 86.4 & 70.7 & 82.4 & 89 & 89.4 & 58.2 \\
MeZO & 92.2 & 74.4 & 69.6 & 76.7 & 57.8 & 77.6 & 88.0 & \textbf{86.7} & 57.1 \\
FZOO & 93.0 & 76.6 & 68.6 & 81.2 & 59.4 & 77.6 & \textbf{89.0} & 86.0 & 57.4 \\
\rowcolor{grzopink}
\textbf{GRZO (Ours)} & \textbf{93.4} & \textbf{81.6} & \textbf{72.0} & \textbf{81.4} & \textbf{59.8} & \textbf{78.6} & \textbf{89.0} & 86.2 & \textbf{65.0} \\
\midrule
\rowcolor{optgreen}
\multicolumn{10}{c}{\textbf{OPT-13B}} \\
\forow Adam (FO) & 95.3 & 80.9 & 94.6 & 83.5 & 66.3 & 76.2 & 88 & 89.5 & 31.3 \\
\forow LoRA (FO) & 94.8 & 78.3 & 69.6 & 80.2 & 64.3 & 69.4 & 89 & 88.0 & 30.9 \\
MeZO & 91.4 & 66.1 & 66.0 & 67.6 & \textbf{59.4} & 57.3 & \textbf{88.0} & 84.7 & 30.9 \\
FZOO & \textbf{93.8} & 76.8 & 69.6 & \textbf{72.2} & \textbf{59.4} & 57.6 & 87.0 & 84.8 & 28.7 \\
\rowcolor{grzopink}
\textbf{GRZO (Ours)} & 93.4 & \textbf{78.0} & \textbf{70.2} & 70.4 & 58.6 & \textbf{57.8} & \textbf{88.0} & \textbf{85.2} & \textbf{32.8} \\
\bottomrule
\end{tabular}
\end{table*}

\paragraph{Auto-Regressive Language Models.}
We expand to Llama3-8B and OPT-13B on the same SuperGLUE+QA suite. Table~\ref{tab:main-llm-table} shows GRZO is the best ZO method on 8/9 Llama3-8B tasks (average 78.6 vs.\ 76.5 FZOO and 75.6 MeZO, $+2.0$ under the same two-forward-pass budget) and 6/9 OPT-13B tasks (average 70.5 vs.\ 70.0 FZOO and 67.9 MeZO). The largest gains over FZOO appear on tasks requiring deeper understanding: $+3.4$ CB, $+5.0$ RTE, $+7.6$ DROP on Llama3-8B; $+4.1$ DROP, $+1.2$ RTE, $+1.0$ COPA on OPT-13B.

\begin{figure*}[!t]
\centering
\includegraphics[width=\textwidth]{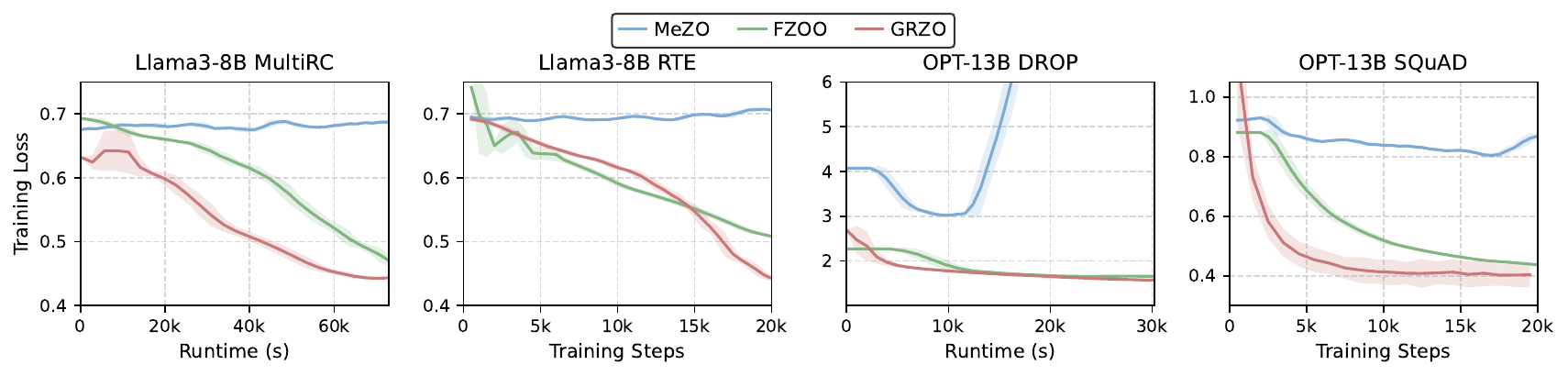}
\caption{Training-loss curves on Llama3-8B (RTE, MultiRC) and OPT-13B (SQuAD, DROP) plotted against training steps and wall-clock time.}
\label{fig:convergence}
\end{figure*}

\paragraph{Convergence and Wall-Clock Efficiency.}
Figure~\ref{fig:convergence} plots training-loss curves on Llama3-8B (RTE, MultiRC) and OPT-13B (SQuAD, DROP) against both training steps and wall-clock seconds. GRZO descends faster than FZOO and MeZO across all four panels and reaches a lower final loss on Llama-RTE, Llama-MultiRC, and OPT-SQuAD; on Llama-MultiRC in particular, GRZO matches the final loss of FZOO in roughly half the wall-clock time. MeZO fails to descend on OPT-DROP and makes no measurable progress on Llama-RTE or Llama-MultiRC.

\subsection{Memory and Time Analysis}
\label{sec:memory-analysis}

GRZO holds inference-level memory at a modest per-step time cost. Figure~\ref{fig:profile} reports a production profile on Llama3-8B; Figure~\ref{fig:gpu-memory} extends the memory picture across model sizes.
\begin{figure}[t]
    \includegraphics[width=\columnwidth]{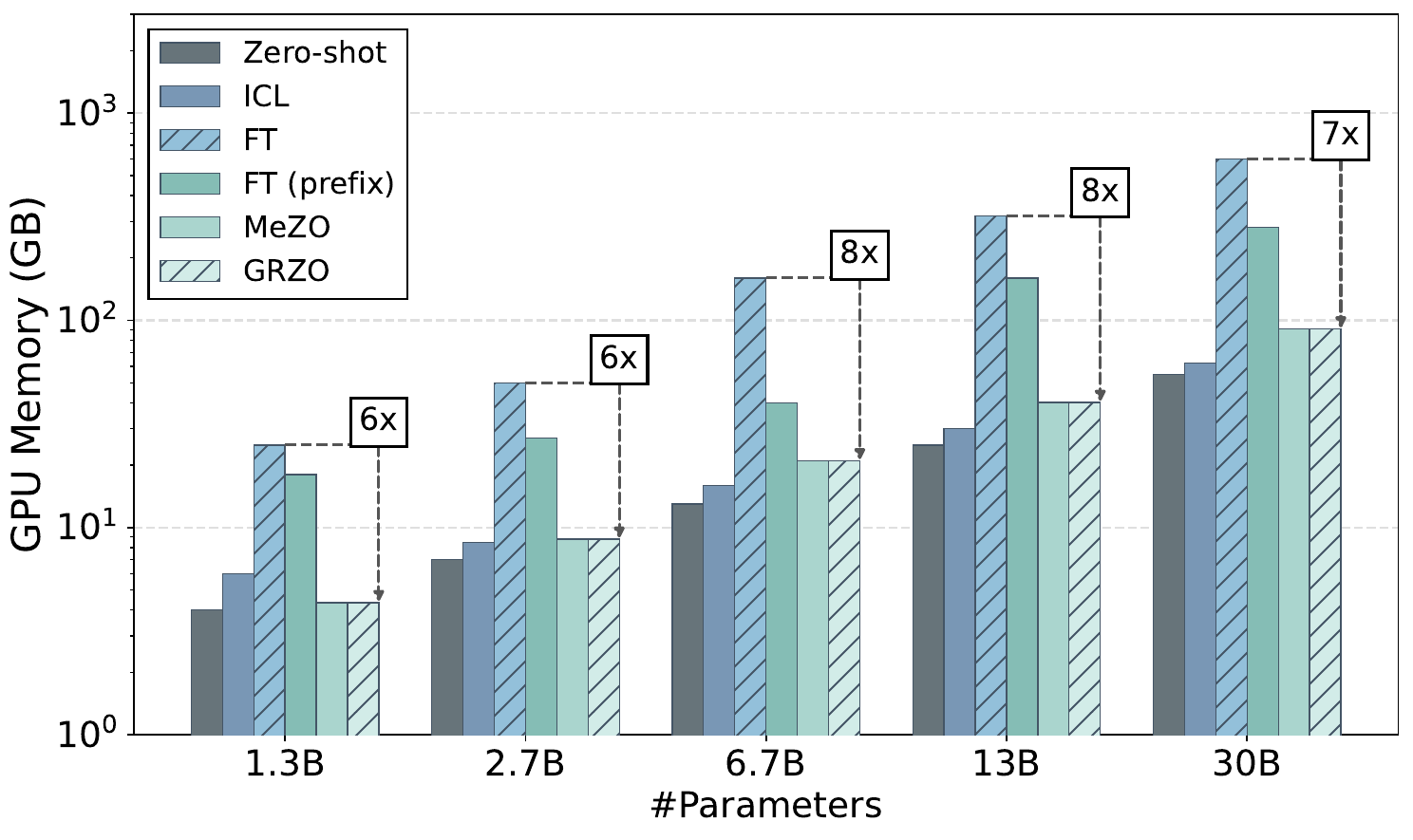}
    \caption{Peak GPU memory (GB) vs.\ model size for OPT (1.3B--30B). GRZO tracks the forward-only inference footprint across scales, so the $B$ per-example directions are obtained without departing from inference-level memory.}
    \label{fig:gpu-memory}
    \vspace{-10pt}
\end{figure}

\begin{figure*}[t]
\centering
\begin{subfigure}[b]{0.49\textwidth}
    \centering
    \includegraphics[width=\linewidth]{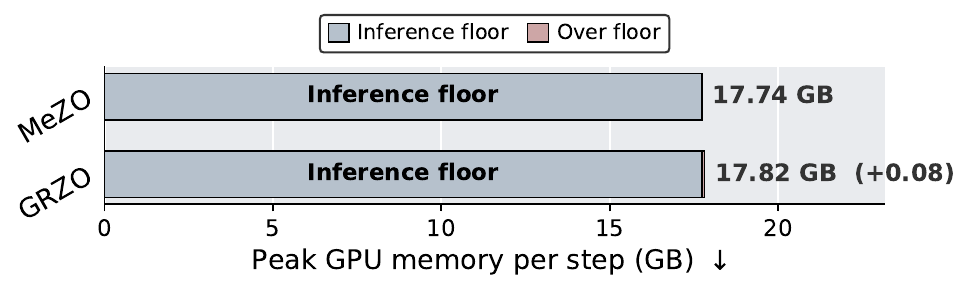}
    \caption{Peak GPU memory breakdown.}
    \label{fig:memory-profile}
\end{subfigure}
\hfill
\begin{subfigure}[b]{0.49\textwidth}
    \centering
    \includegraphics[width=\linewidth]{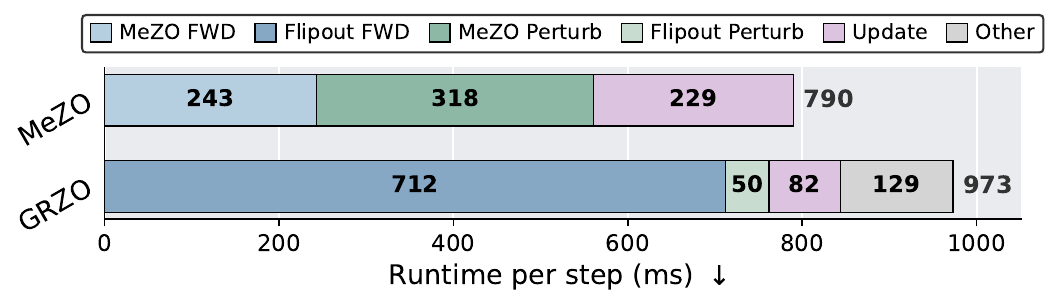}
    \caption{Per-step time breakdown.}
    \label{fig:time-breakdown}
\end{subfigure}
\caption{Production profile on Llama3-8B (RTE, fp16, $B{=}16$, 4$\times$A100, mean over 20 steps $\times$ 4 ranks). \textbf{Left}: peak GPU memory per step, split into the forward-only inference floor and the overhead above it. \textbf{Right}: per-step time decomposition.}
\label{fig:profile}
\end{figure*}

\paragraph{Memory.}
We take the \emph{forward-only inference peak} as the reference point, since no forward-only optimizer can go below it; on Llama3-8B (RTE, fp16, $B{=}16$) it is $17.74$~GB, which already includes activations and workspace on top of the $16.0$~GB of weights. Fig.~\ref{fig:memory-profile} shows that vanilla GRZO peaks at $17.82$~GB, i.e.\ $+0.08$~GB or $0.5\%$ above this floor, so extracting $B$ per-example directions costs essentially nothing in memory. A memory-optimal in-place MeZO also sits at the floor ($17.74$~GB, $+0.00$~GB): GRZO does not use \emph{less} memory than MeZO, it obtains $B$ directions instead of one at the same footprint. We profile MeZO with a fully in-place perturb--restore--update implementation so that it attains its true inference-level optimum. The picture carries over to the combined variants (Section~\ref{sec:composing}): the GRZO family stays within $0.7\%$ of the floor wherever the paired variant does, and where a variant adds its own scratch storage---Sparse and quantized perturbations---GRZO is the cheaper of the pair by $0.3$ and $1.5$~GB, the latter within the fake-quant regime discussed in Appendix~\ref{app:profile_breakdown}. Figure~\ref{fig:gpu-memory} shows that GRZO tracks the inference footprint across OPT 1.3B--30B, staying $6$--$8\times$ below full fine-tuning.

\paragraph{Per-Step Time.}
Figure~\ref{fig:time-breakdown} decomposes the per-step wall-clock cost. A fully in-place MeZO completes in $790$~ms; GRZO fuses the per-example perturbation into the forward via per-Linear pre-hooks, making this fused forward $\sim$27\% slower than MeZO's forward-plus-in-place-perturbation ($712$ vs.\ $561$~ms). The update shrinks from $229$ to $82$~ms via sign-vector products, and GRZO totals $973$~ms ($+23\%$ over MeZO). The added forward work is one GEMM per Linear layer of the same shape as that layer's own matmul, plus two elementwise sign products; since both scale with the layer's own cost, the relative overhead is largely invariant to batch size, hidden dimension, and model size. This per-step cost is more than offset by the variance reduction (factor of $1/B_{\mathrm{eff}}$, Theorem~\ref{thm:grzo_var}), which translates into proportionally fewer optimization steps to a target loss (Figure~\ref{fig:convergence}). The same trade applies to combined variants; full per-method breakdowns are in Appendix~\ref{app:profile_breakdown}.

\begin{figure*}[!t]
\centering
{\small
\captionof{table}{GRZO with orthogonal ZO baselines on Llama3-8B. Parentheses: $(\Delta$ vs.\ paired baseline$\,/\,\Delta$ vs.\ vanilla GRZO$)$. Best per task in \textbf{bold}. QuZO and Qu-GRZO both use 8-bit quantization for weights and perturbations.}
\label{tab:additional-baselines}
\setlength{\tabcolsep}{3pt}
\renewcommand{\arraystretch}{1.2}
\begin{tabular}{@{}lccccc@{}}
\toprule
\rowcolor{tableheader}
\textbf{Method} & \multicolumn{3}{c}{\textbf{SuperGLUE (Classification)}} & \multicolumn{2}{c}{\textbf{QA (F1)}} \\
\cmidrule(lr){2-4}\cmidrule(lr){5-6}
\rowcolor{tableheader}
 & \textbf{BoolQ} & \textbf{RTE} & \textbf{COPA} & \textbf{SQuAD} & \textbf{DROP} \\
\midrule
\multicolumn{6}{@{}l}{\textit{ZO variants baselines}} \\
Sparse-MeZO~\citep{liu2024sparse} & 80.5 & 73.4 & 83.0 & 87.5 & 48.4 \\
LOZO~\citep{chen2025enhancing}        & 79.4 & 72.1 & 84.0 & \textbf{89.0} & 65.4 \\
QuZO (\text{int}8)~\citep{zhou2025quzo}      & 76.8 & 75.2 & 87.0 & 80.6 & 52.3 \\
\midrule
\multicolumn{6}{@{}l}{\textit{Vanilla GRZO}} \\
\rowcolor{grzopink}
\textbf{GRZO} & 81.4 & \textbf{81.6} & 89.0 & 86.2 & 65.0 \\
\midrule
\multicolumn{6}{@{}l}{\textit{GRZO combined with orthogonal ZO variants}} \\
\rowcolor{grzopinkdeep}
Sparse-GRZO & \textbf{85.1} (\dpos{+4.6}/\dpos{+3.7}) & 79.4 (\dpos{+6.0}/\dneg{-2.2}) & 88.0 (\dpos{+5.0}/\dneg{-1.0}) & \textbf{89.0} (\dpos{+1.5}/\dpos{+2.8}) & 59.3 (\dpos{+10.9}/\dneg{-5.7}) \\
\rowcolor{grzopinkdeep}
LO-GRZO   & 84.4 (\dpos{+5.0}/\dpos{+3.0}) & 75.1 (\dpos{+3.0}/\dneg{-6.5}) & 90.0 (\dpos{+6.0}/\dpos{+1.0}) & 88.4 (\dneg{-0.6}/\dpos{+2.2}) & \textbf{65.5} (\dpos{+0.1}/\dpos{+0.5}) \\
\rowcolor{grzopinkdeep}
Qu-GRZO (\text{int}8)   & 79.3 (\dpos{+2.5}/\dneg{-2.1}) & 80.5 (\dpos{+5.3}/\dneg{-1.1}) & \textbf{91.0} (\dpos{+4.0}/\dpos{+2.0}) & 88.6 (\dpos{+8.0}/\dpos{+2.4}) & 63.9 (\dpos{+11.6}/\dneg{-1.1}) \\
\bottomrule
\end{tabular}
}

\vspace{10pt}

\includegraphics[width=\textwidth]{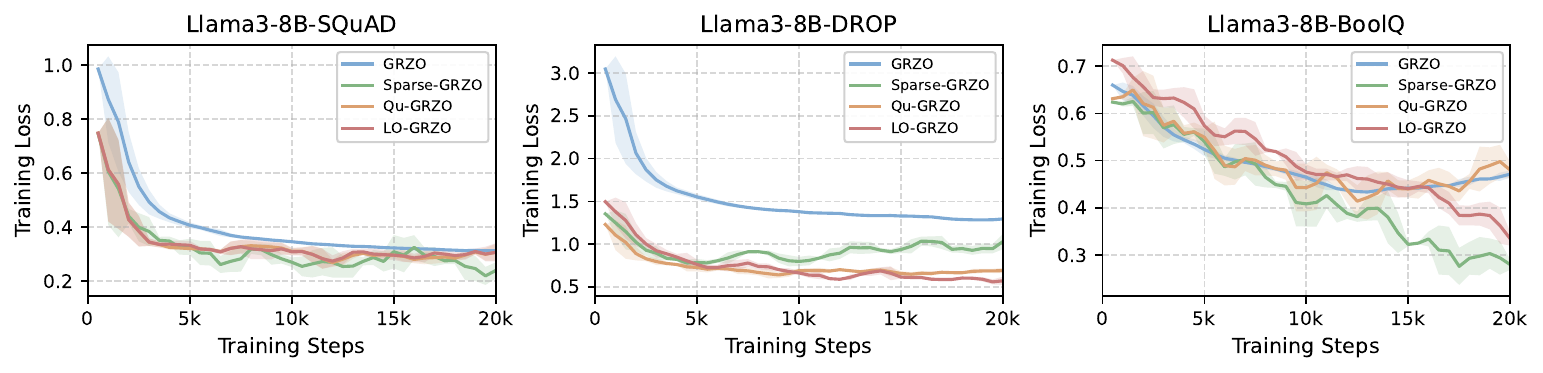}
\caption{Training loss curves on Llama3-8B comparing vanilla GRZO with the three GRZO-combined variants.}
\label{fig:combine-convergence}
\end{figure*}

\paragraph{Why $B$ Directions Are Memory-Free.}
The question is not why GRZO is cheaper than MeZO---at their memory optima both sit at the inference floor---but why moving from one perturbation direction to $B$ does not multiply the footprint. Two mechanisms are responsible. First, GRZO never materializes a perturbed weight tensor: each layer's perturbed weight is built transiently inside a forward pre-hook and freed before the next layer runs, so at most one layer's worth of perturbation is live at any moment, independent of depth. Second, the sign factorization $\Delta\mat{W}_i=\mat{U}\odot(\mat{r}_i\mat{s}_i^\top)$ compresses all per-example variation into $\pm1$ sign vectors of size $B(d_{\text{out}}{+}d_{\text{in}})$ per layer, so the $B$-direction estimator carries the same per-layer state as a single-direction one. The residual $+0.08$~GB is exactly these sign vectors and the transient per-layer buffer.

\subsection{Combination with ZO Variants}
\label{sec:composing}

\noindent\textbf{GRZO as a drop-in base for other ZO variants.}
We frame GRZO not only as a stronger ZO estimator than MeZO, but as a base on which existing MeZO variants---sparse (Sparse-MeZO), low-rank (LOZO), and quantized (QuZO) perturbations---can be dropped in for further gains. GRZO intervenes only at the MeZO core, swapping the single-direction perturbation and scalar update for multi-directional per-example perturbation and group-relative normalization; sparsity, low-rank, and quantization act on orthogonal axes of the perturbation and remain complementary to GRZO.

Table~\ref{tab:additional-baselines} shows the composition is Pareto-favorable along two directions. First, every GRZO-combined variant outperforms its baseline on every shared task (the sole exception is LO-GRZO at $-0.6$ F1 on SQuAD); the most striking case is Qu-GRZO on DROP, where $52.3 \to 63.9$ F1 ($+11.6$) brings the low-bit baseline within $1.1$ F1 of vanilla GRZO ($65.0$). The quantized pairing benefits twice: on top of the $1/B$ variance reduction, low-bit forward passes inflate the scale of the loss differences, and a MeZO-style update---being proportional to $\delta$---absorbs that inflation as an uncontrolled increase in effective step size, whereas the group-relative weights divide by an $s$ inflated by the same factor and are therefore unaffected (Section~\ref{sec:grzo_core}). The gain thus reflects GRZO's scale invariance rather than any special interaction with quantization; Appendix~\ref{app:quant_scale} measures both effects. Second, the combinations also beat \emph{vanilla} GRZO on 8 of 15 task--variant cells, consistently so on SQuAD and most reliably for LO-GRZO, showing that these ZO variants complement, rather than compete with, GRZO. Fig.~\ref{fig:combine-convergence} confirms both effects in the training dynamics; two per-variant comparison examples on SQuAD and BoolQ are in Figure~\ref{fig:per-variant-compare} (Appendix~\ref{app:per_variant_compare}). The standard ZO taxonomy---sparsity, low-rank, quantization, and variance reduction---thus reads as independent axes that can be stacked, with GRZO supplying the variance-reduction axis that has no native solution at MeZO's forward budget.

\paragraph{Efficiency Benefit of the GRZO Core.}
Swapping the MeZO core for GRZO inside any ZO variant improves not only convergence quality but also resource efficiency. Across the three fp16 variant pairings (vanilla, LOZO, Sparse), GRZO+X is memory-neutral relative to MeZO+X---$+0.08$~GB on vanilla and LOZO, and $0.33$~GB lower on Sparse---so the $B$ per-example directions come at no memory premium. GRZO+X is 6--23\% slower per step than MeZO+X---each forward fuses the per-example perturbation---but the variance reduction greatly reduces the number of optimization steps to a target loss (Figure~\ref{fig:combine-convergence}). We exclude QuZO and Qu-GRZO from this efficiency claim because our fp16 implementations measure fake-quant overhead rather than the low-bit deployment regime targeted by~\citet{zhou2025quzo}; per-method numbers and a full discussion are in Appendix~\ref{app:profile_breakdown}.

\subsection{Ablation Study}
\label{sec:ablation}

Two ablations isolate GRZO's design: which component---per-example (PE) perturbation vs.\ group-relative normalization (GN)---drives convergence, and whether the choice of perturbation noise distribution materially affects performance.

\paragraph{Component Ablation.}
Figure~\ref{fig:ablation} (left) compares full GRZO, GRZO without GN, and GRZO without both PE and GN (a MeZO-style estimator) on SST-2. Removing GN alone slows descent on this task; further removing PE degrades it more, showing the two components are complementary---PE enables efficient batched estimation, GN stabilizes the gradient signal. GRZO beats FZOO by $+0.4$ on SST-2 ($+2.0$ on Llama3-8B average; Section~\ref{sec:lm_results}), ruling out a normalization-only explanation.

\begin{figure}[t]
\centering
\includegraphics[width=\columnwidth]{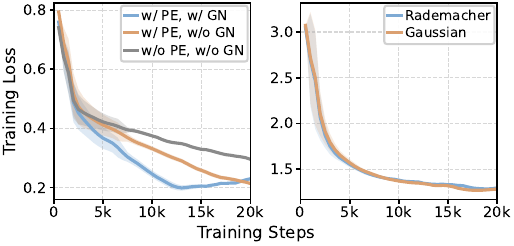}
\caption{\textbf{Left}: GRZO components on SST-2. \textbf{Right}: Perturbation ablation on DROP.}
\label{fig:ablation}
\end{figure}

\paragraph{Perturbation Type.}
Figure~\ref{fig:ablation} (right) compares Gaussian and Rademacher perturbations on DROP; the two loss curves are nearly indistinguishable, confirming robustness to the choice of base distribution. We default to Rademacher for its lower memory overhead and because it is the choice for which the shared-base variance term of Theorem~\ref{thm:grzo_var} vanishes exactly ($m_4{=}1$).

\paragraph{Batch Size Sensitivity.}
\label{sec:grzo_bs_sensitivity}
Figure~\ref{fig:bs-ablation} (Appendix~\ref{app:bs-ablation}) shows $B{=}4$ diverges and $B{=}8$ is unstable, while $B{\geq}16$ converges smoothly. This matches Proposition~\ref{prop:grzo_norm}: group-relative normalization perturbs the estimator by a relative $O(1/B)$, which is negligible at $B{\geq}16$ but no longer so once $B$ falls to $4$. We therefore recommend $B \geq 16$.

\section{Conclusion}

We have presented GRZO, a zeroth-order optimizer that treats the mini-batch as a source of perturbation directions rather than only loss-averaging samples. By combining pseudo-independent per-example perturbations with group-relative normalization, GRZO extracts a gradient direction from every mini-batch example under MeZO's two-forward-pass budget, greatly mitigating the variance bottleneck of single-direction ZO methods. We have provided theoretical guarantees on directional unbiasedness, variance reduction, and nonconvex convergence. Extensive experiments on multiple models have shown that GRZO outperforms state-of-the-art ZO baselines on the large majority of tasks at inference-level memory, and serves as a drop-in replacement for the MeZO core that composes with sparsity, low-rank, and quantization variants. Extending GRZO to further MeZO variants and full pre-training of large-scale models remains open.

\section*{Limitations}
\label{sec:limitations}

GRZO's main per-step cost comes from running two perturbed forwards ($\ell^+$ and $\ell^-$) to form a directionally unbiased two-sided estimate. A one-sided variant comparing a single perturbed forward to an unperturbed forward would be faster than MeZO but at the cost of a biased gradient estimate and slower descent; building a low-bias one-sided design is open future work.

Our empirical study covers four ZO families (vanilla, Sparse-MeZO, LOZO, QuZO) up to 13B parameters. Other MeZO variants---curvature-preconditioned (HiZOO), control-variate (MeZO-SVRG), and subspace-orthogonalization---and 70B+ scaling are the most immediate empirical extensions: the GRZO core swap is mechanically straightforward in each case, but the variance-versus-stability trade-off and downstream quality at those settings remain to be verified.

\section*{Ethical Considerations}
\label{sec:ethics}

This work uses publicly available pre-trained models and benchmarks for research purposes, follows their respective licenses and terms of use, and does not involve the collection or release of personally identifiable information. We do not foresee specific ethical concerns beyond the general risks associated with large language models. In particular, GRZO may inherit biases, hallucinations, or misleading patterns from the underlying models and data. We therefore do not recommend deploying it as a standalone decision-making system, especially in high-stakes settings. Any practical use should include human oversight and task-specific safety evaluation.

\bibliography{custom}

@inproceedings{chen2025enhancing,
  title={Enhancing Zeroth-Order Fine-Tuning for Language Models with Low-Rank Structures},
  author={Chen, Yiming and Zhang, Yuan and Cao, Liyuan and Yuan, Kun and Wen, Zaiwen},
  booktitle={International Conference on Learning Representations},
  volume={2025},
  pages={62581--62607},
  year={2025}
}

@article{malladi2023fine,
  title={Fine-tuning language models with just forward passes},
  author={Malladi, Sadhika and Gao, Tianyu and Nichani, Eshaan and Damian, Alex and Lee, Jason D and Chen, Danqi and Arora, Sanjeev},
  journal={Advances in Neural Information Processing Systems},
  volume={36},
  pages={53038--53075},
  year={2023}
}

@article{amari1993backpropagation,
  title={Backpropagation and stochastic gradient descent method},
  author={Amari, Shun-ichi},
  journal={Neurocomputing},
  volume={5},
  number={4-5},
  pages={185--196},
  year={1993},
  publisher={Elsevier}
}

@inproceedings{loshchilov2017decoupled,
  title={Decoupled weight decay regularization},
  author={Loshchilov, Ilya and Hutter, Frank},
  booktitle={International Conference on Learning Representations},
  year={2019}
}

@inproceedings{gautam2024mezosvrg,
  title={Variance-reduced Zeroth-Order Methods for Fine-Tuning Language Models},
  author={Gautam, Tanmay and Park, Youngsuk and Zhou, Hao and Raman, Parameswaran and Ha, Wooseok},
  booktitle={Proceedings of the 41st International Conference on Machine Learning},
  series={Proceedings of Machine Learning Research},
  volume={235},
  pages={15180--15208},
  year={2024},
  publisher={PMLR}
}

@inproceedings{zhao2024hizoo,
  title={Second-Order Fine-Tuning without Pain for {LLM}s: A {H}essian Informed Zeroth-Order Optimizer},
  author={Zhao, Yanjun and Dang, Sizhe and Ye, Haishan and Dai, Guang and Qian, Yi and Tsang, Ivor W.},
  booktitle={International Conference on Learning Representations},
  year={2025}
}

@inproceedings{yang2024adazeta,
  title={{AdaZeta}: Adaptive zeroth-order tensor-train adaption for memory-efficient large language models fine-tuning},
  author={Yang, Yifan and Zhen, Kai and Banijamali, Ershad and Mouchtaris, Athanasios and Zhang, Zheng},
  booktitle={Proceedings of the 2024 Conference on Empirical Methods in Natural Language Processing},
  pages={977--995},
  year={2024}
}

@inproceedings{zhang2025mazo,
  title={{MaZO}: Masked zeroth-order optimization for multi-task fine-tuning of large language models},
  author={Zhang, Zhen and Yang, Yifan and Zhen, Kai and Susanj, Nathan and Mouchtaris, Athanasios and Kunzmann, Siegfried and Zhang, Zheng},
  booktitle={Proceedings of the 2025 Conference on Empirical Methods in Natural Language Processing},
  pages={18537--18554},
  year={2025}
}

@inproceedings{zhou2025quzo,
  title={{QuZO}: Quantized Zeroth-Order Fine-Tuning for Large Language Models},
  author={Zhou, Jiajun and Yang, Yifan and Zhen, Kai and Liu, Ziyue and Zhao, Yequan and Banijamali, Ershad and Mouchtaris, Athanasios and Wong, Ngai and Zhang, Zheng},
  booktitle={Proceedings of the 2025 Conference on Empirical Methods in Natural Language Processing},
  pages={5341--5359},
  year={2025},
  doi={10.18653/v1/2025.emnlp-main.271}
}

@inproceedings{liu2024sparse,
  title={Sparse {MeZO}: Less Parameters for Better Performance in Zeroth-Order {LLM} Fine-Tuning},
  author={Liu, Yong and Zhu, Zirui and Gong, Chaoyu and Cheng, Minhao and Hsieh, Cho-Jui and You, Yang},
  booktitle={Advances in Neural Information Processing Systems},
  year={2025}
}

@inproceedings{wen2018flipout,
  title={Flipout: Efficient pseudo-independent weight perturbations on mini-batches},
  author={Wen, Yeming and Vicol, Paul and Ba, Jimmy and Tran, Dustin and Grosse, Roger},
  booktitle={International Conference on Learning Representations},
  year={2018}
}

@inproceedings{kingma2014adam,
  title={{Adam}: A method for stochastic optimization},
  author={Kingma, Diederik P. and Ba, Jimmy},
  booktitle={International Conference on Learning Representations},
  year={2015}
}

@inproceedings{hu2022lora,
  title={{LoRA}: {L}ow-{R}ank adaptation of large language models},
  author={Hu, Edward J and Shen, Yelong and Wallis, Phillip and Allen-Zhu, Zeyuan and Li, Yuanzhi and Wang, Shean and Wang, Lu and Chen, Weizhu},
  booktitle={International Conference on Learning Representations},
  year={2022}
}

@inproceedings{dang2025fzoo,
  title={{FZOO}: Fast Zeroth-Order Optimizer for Fine-Tuning Large Language Models towards {A}dam-Scale Speed},
  author={Dang, Sizhe and Guo, Yangyang and Zhao, Yanjun and Ye, Haishan and Zheng, Xiaodong and Dai, Guang and Tsang, Ivor W.},
  booktitle={International Conference on Learning Representations},
  year={2026}
}

@article{wang2019superglue,
  title={{SuperGLUE}: A stickier benchmark for general-purpose language understanding systems},
  author={Wang, Alex and Pruksachatkun, Yada and Nangia, Nikita and Singh, Amanpreet and Michael, Julian and Hill, Felix and Levy, Omer and Bowman, Samuel},
  journal={Advances in neural information processing systems},
  volume={32},
  year={2019}
}

@article{grattafiori2024llama,
  title={The {Llama} 3 herd of models},
  author={Grattafiori, Aaron and Dubey, Abhimanyu and Jauhri, Abhinav and Pandey, Abhinav and Kadian, Abhishek and Al-Dahle, Ahmad and Letman, Aiesha and Mathur, Akhil and Schelten, Alan and Vaughan, Alex and others},
  journal={arXiv preprint arXiv:2407.21783},
  year={2024}
}

@article{zhang2022opt,
  title={{OPT}: {O}pen {P}re-trained transformer language models},
  author={Zhang, Susan and Roller, Stephen and Goyal, Naman and Artetxe, Mikel and Chen, Moya and Chen, Shuohui and Dewan, Christopher and Diab, Mona and Li, Xian and Lin, Xi Victoria and others},
  journal={arXiv preprint arXiv:2205.01068},
  year={2022}
}

@inproceedings{zhao2024galorememoryefficientllmtraining,
  title={{GaLore}: Memory-Efficient {LLM} Training by Gradient Low-Rank Projection},
  author={Zhao, Jiawei and Zhang, Zhenyu and Chen, Beidi and Wang, Zhangyang and Anandkumar, Anima and Tian, Yuandong},
  booktitle={International Conference on Machine Learning},
  year={2024}
}

@inproceedings{devlin2019bert,
  title={{BERT}: {P}re-training of deep bidirectional transformers for language understanding},
  author={Devlin, Jacob and Chang, Ming-Wei and Lee, Kenton and Toutanova, Kristina},
  booktitle={Proceedings of the 2019 conference of the North American chapter of the association for computational linguistics: human language technologies, volume 1 (long and short papers)},
  pages={4171--4186},
  year={2019}
}

@article{liu2019roberta,
  title={{RoBERTa}: A robustly optimized {BERT} pretraining approach},
  author={Liu, Yinhan and Ott, Myle and Goyal, Naman and Du, Jingfei and Joshi, Mandar and Chen, Danqi and Levy, Omer and Lewis, Mike and Zettlemoyer, Luke and Stoyanov, Veselin},
  journal={arXiv preprint arXiv:1907.11692},
  year={2019}
}

@article{lang2026powering,
  title={Powering Up Zeroth-Order Training via Subspace Gradient Orthogonalization},
  author={Lang, Yicheng and Wang, Changsheng and Zhang, Yihua and Hong, Mingyi and Zhang, Zheng and Yin, Wotao and Liu, Sijia},
  journal={arXiv preprint arXiv:2602.17155},
  year={2026}
}

@inproceedings{ma2025revisiting,
  title={Revisiting Zeroth-Order Optimization: Minimum-Variance Two-Point Estimators and Directionally Aligned Perturbations},
  author={Ma, Shaocong and Huang, Heng},
  booktitle={International Conference on Learning Representations},
  volume={2025},
  pages={13267--13293},
  year={2025}
}

@article{duchi2015optimal,
  title={Optimal rates for zero-order convex optimization: The power of two function evaluations},
  author={Duchi, John C and Jordan, Michael I and Wainwright, Martin J and Wibisono, Andre},
  journal={IEEE Transactions on Information Theory},
  volume={61},
  number={5},
  pages={2788--2806},
  year={2015},
  publisher={IEEE}
}

@article{ghadimi2013stochastic,
  title={Stochastic first- and zeroth-order methods for nonconvex stochastic programming},
  author={Ghadimi, Saeed and Lan, Guanghui},
  journal={SIAM journal on optimization},
  volume={23},
  number={4},
  pages={2341--2368},
  year={2013},
  publisher={SIAM}
}

@article{spall2002multivariate,
  title={Multivariate stochastic approximation using a simultaneous perturbation gradient approximation},
  author={Spall, James C},
  journal={IEEE transactions on automatic control},
  volume={37},
  number={3},
  pages={332--341},
  year={2002},
  publisher={IEEE}
}

@article{shao2024deepseekmath,
  title={{DeepSeekMath}: Pushing the limits of mathematical reasoning in open language models},
  author={Shao, Zhihong and Wang, Peiyi and Zhu, Qihao and Xu, Runxin and Song, Junxiao and Bi, Xiao and Zhang, Haowei and Zhang, Mingchuan and Li, YK and Wu, Yang and others},
  journal={arXiv preprint arXiv:2402.03300},
  year={2024}
}

@article{zhao2025poor,
  title={Poor man's training on {MCU}s: A memory-efficient quantized back-propagation-free approach},
  author={Zhao, Yequan and Li, Hai and Young, Ian and Zhang, Zheng},
  journal={ACM Transactions on Design Automation of Electronic Systems},
  volume={30},
  number={5},
  pages={1--33},
  year={2025},
  publisher={ACM New York, NY}
}

@article{liu2020primer,
  title={A primer on zeroth-order optimization in signal processing and machine learning: Principals, recent advances, and applications},
  author={Liu, Sijia and Chen, Pin-Yu and Kailkhura, Bhavya and Zhang, Gaoyuan and Hero III, Alfred O and Varshney, Pramod K},
  journal={IEEE Signal Processing Magazine},
  volume={37},
  number={5},
  pages={43--54},
  year={2020},
  publisher={IEEE}
}

@article{nesterov2017random,
  title={Random gradient-free minimization of convex functions},
  author={Nesterov, Yurii and Spokoiny, Vladimir},
  journal={Foundations of Computational Mathematics},
  volume={17},
  number={2},
  pages={527--566},
  year={2017},
  publisher={Springer}
}

@inproceedings{gao2022generalizing,
  title={Generalizing gaussian smoothing for random search},
  author={Gao, Katelyn and Sener, Ozan},
  booktitle={International Conference on Machine Learning},
  pages={7077--7101},
  year={2022},
  organization={PMLR}
}

@inproceedings{malladi2023kernel,
  title={A kernel-based view of language model fine-tuning},
  author={Malladi, Sadhika and Wettig, Alexander and Yu, Dingli and Chen, Danqi and Arora, Sanjeev},
  booktitle={International Conference on Machine Learning},
  pages={23610--23641},
  year={2023},
  organization={PMLR}
}

@inproceedings{aghajanyan2021intrinsic,
  title={Intrinsic dimensionality explains the effectiveness of language model fine-tuning},
  author={Aghajanyan, Armen and Gupta, Sonal and Zettlemoyer, Luke},
  booktitle={Proceedings of the 59th annual meeting of the association for computational linguistics and the 11th international joint conference on natural language processing (volume 1: long papers)},
  pages={7319--7328},
  year={2021}
}

@article{berahas2022theoretical,
  title={A theoretical and empirical comparison of gradient approximations in derivative-free optimization},
  author={Berahas, Albert S and Cao, Liyuan and Choromanski, Krzysztof and Scheinberg, Katya},
  journal={Foundations of Computational Mathematics},
  volume={22},
  number={2},
  pages={507--560},
  year={2022},
  publisher={Springer}
}

@inproceedings{chen2024deepzero,
  title={{DeepZero}: Scaling up zeroth-order optimization for deep model training},
  author={Chen, Aochuan and Zhang, Yimeng and Jia, Jinghan and Diffenderfer, James and Parasyris, Konstantinos and Liu, Jiancheng and Zhang, Yihua and Zhang, Zheng and Kailkhura, Bhavya and Liu, Sijia},
  booktitle={International Conference on Learning Representations},
  volume={2024},
  pages={50185--50206},
  year={2024}
}

@article{zhao2023tensor,
  title={Tensor-compressed back-propagation-free training for (physics-informed) neural networks},
  author={Zhao, Yequan and Yu, Xinling and Chen, Zhixiong and Liu, Ziyue and Liu, Sijia and Zhang, Zheng},
  journal={arXiv preprint arXiv:2308.09858},
  year={2023}
}

@inproceedings{wang2018glue,
  title={{GLUE}: A multi-task benchmark and analysis platform for natural language understanding},
  author={Wang, Alex and Singh, Amanpreet and Michael, Julian and Hill, Felix and Levy, Omer and Bowman, Samuel},
  booktitle={Proceedings of the 2018 {EMNLP} Workshop {BlackboxNLP}: Analyzing and Interpreting Neural Networks for {NLP}},
  pages={353--355},
  year={2018}
}

@inproceedings{rajpurkar2016squad,
  title={{SQuAD}: 100,000+ questions for machine comprehension of text},
  author={Rajpurkar, Pranav and Zhang, Jian and Lopyrev, Konstantin and Liang, Percy},
  booktitle={Proceedings of the 2016 conference on empirical methods in natural language processing},
  pages={2383--2392},
  year={2016}
}

@inproceedings{dua2019drop,
  title={{DROP}: A reading comprehension benchmark requiring discrete reasoning over paragraphs},
  author={Dua, Dheeru and Wang, Yizhong and Dasigi, Pradeep and Stanovsky, Gabriel and Singh, Sameer and Gardner, Matt},
  booktitle={Proceedings of the 2019 Conference of the North American Chapter of the Association for Computational Linguistics: Human Language Technologies, Volume 1 (Long and Short Papers)},
  pages={2368--2378},
  year={2019}
}

@inproceedings{liu2025cola,
  title={{CoLA}: Compute-efficient pre-training of {LLM}s via low-rank activation},
  author={Liu, Ziyue and Zhang, Ruijie and Wang, Zhengyang and Yan, Mingsong and Yang, Zi and Hovland, Paul D and Nicolae, Bogdan and Cappello, Franck and Tang, Sui and Zhang, Zheng},
  booktitle={Proceedings of the 2025 Conference on Empirical Methods in Natural Language Processing},
  pages={4627--4645},
  year={2025}
}

@article{zhang2026lax,
  title={{LaX}: Boosting low-rank training of foundation models via latent crossing},
  author={Zhang, Ruijie Ray and Liu, Ziyue Alvin and Wang, Zhengyang and Zhang, Zheng},
  journal={Advances in Neural Information Processing Systems},
  volume={38},
  pages={142920--142948},
  year={2025}
}

@article{yang2026sharpzo,
  title={{SharpZO}: Hybrid sharpness-aware vision language model prompt tuning via forward-only passes},
  author={Yang, Yifan and Zhang, Zhen and Swaminathan, Rupak Vignesh and Liu, Jing and Susanj, Nathan and Zhang, Zheng},
  journal={Advances in Neural Information Processing Systems},
  volume={38},
  pages={143695--143721},
  year={2025}
}

@inproceedings{li2021prefix,
  title={Prefix-tuning: Optimizing continuous prompts for generation},
  author={Li, Xiang Lisa and Liang, Percy},
  booktitle={Proceedings of the 59th Annual Meeting of the Association for Computational Linguistics and the 11th International Joint Conference on Natural Language Processing (Volume 1: Long Papers)},
  pages={4582--4597},
  year={2021}
}

@inproceedings{houlsby2019parameter,
  title={Parameter-efficient transfer learning for {NLP}},
  author={Houlsby, Neil and Giurgiu, Andrei and Jastrzebski, Stanislaw and Morrone, Bruna and De Laroussilhe, Quentin and Gesmundo, Andrea and Attariyan, Mona and Gelly, Sylvain},
  booktitle={International conference on machine learning},
  pages={2790--2799},
  year={2019},
  organization={PMLR}
}

@inproceedings{lester2021power,
  title={The power of scale for parameter-efficient prompt tuning},
  author={Lester, Brian and Al-Rfou, Rami and Constant, Noah},
  booktitle={Proceedings of the 2021 conference on empirical methods in natural language processing},
  pages={3045--3059},
  year={2021}
}

@article{tan2026zoaf,
  title={{ZOAF}: Towards Efficient Zeroth-Order Optimization for Analog/{RF} Circuit Design},
  author={Tan, Liyan and Zhao, Yequan and Lu, Jinming and Jamroz, Ben F. and Feldman, Ari and Zhang, Zheng},
  journal={arXiv preprint arXiv:2606.02869},
  year={2026}
}

\clearpage
\appendix

\section{GRZO Algorithm}
\label{app:grzo_algorithm}

\begin{algorithm*}[!t]
\caption{GRZO (Group-Relative Zeroth-Order Optimization)}
\label{alg:grzo}
\begin{algorithmic}[1]
\Require Parameters $\boldsymbol{\theta}=\{\mat{W}^{(\ell)}\}$; scale $\sigma$; batch size $B$;
         steps $T$; learning rates $\{\eta_t\}$; $\epsilon>0$
\For{$t = 1, \ldots, T$}
  \State Sample $\mathcal{B}_t$; draw per-example sign vectors $\{(\mat{r}_i^{(\ell)}, \mat{s}_i^{(\ell)})\}$ and layer seeds $\{\textit{seed}^{(\ell)}\}$
  \For{each layer $\ell$} \Comment{\textit{Fused forward: two forward passes, per-example perturbations}}
    \State Regenerate $\mat{U}^{(\ell)}$ from $\textit{seed}^{(\ell)}$;\quad $\mat{P}_i^{(\ell)} \gets ((\mat{X}_i \odot \mat{S}_i^{(\ell)})\,\mat{U}^{(\ell)}) \odot \mat{R}_i^{(\ell)}$
    \State $\mat{X}_i \gets \textsc{Concat}\!\bigl(\phi(\mat{X}_i \mat{W}^{(\ell)}{+}\sigma \mat{P}_i^{(\ell)}),\; \phi(\mat{X}_i \mat{W}^{(\ell)}{-}\sigma \mat{P}_i^{(\ell)})\bigr)$
  \EndFor
  \State $\delta_i \gets \ell_i^+ - \ell_i^-$ \quad for all $i=1,\ldots,B$ \Comment{\textit{Group-relative normalization}}
  \State $s \gets \sqrt{\tfrac{1}{B}\sum_i(\delta_i{-}\bar\delta)^2}$ with $\bar\delta=\tfrac{1}{B}\sum_i\delta_i$;\quad $a_i \gets \delta_i/(s{+}\epsilon)$
  \For{each layer $\ell$} \Comment{\textit{Seed-regenerated weight update}}
    \State Regenerate $\mat{U}^{(\ell)}$ from $\textit{seed}^{(\ell)}$;\quad $\bar{\mat{M}}^{(\ell)} \gets \tfrac{1}{B}\sum_i a_i\, \mat{r}_i^{(\ell)}(\mat{s}_i^{(\ell)})^\top$
    \State $\mat{W}^{(\ell)} \gets \mat{W}^{(\ell)} - \dfrac{\eta_t}{2\sigma}\, \mat{U}^{(\ell)} \odot \bar{\mat{M}}^{(\ell)}$
  \EndFor
\EndFor
\end{algorithmic}
\end{algorithm*}

Algorithm~\ref{alg:grzo} gives the per-step pseudocode for GRZO, fusing the per-example sign-factorized perturbation construction with group-relative normalization. The same random seeds are used for the perturbation in the forward pass and the seed-regenerated weight update, so the per-example perturbations $\mat{P}_i^{(\ell)}$ never have to be materialized between the two phases.

\section{Additional Background on MeZO and Flipout}
\label{app:background_details}

\subsection{MeZO as In-Place Two-Point Zeroth-Order Optimization}

MeZO \citep{malladi2023fine} adapts the classical two-point SPSA estimator to LLM fine-tuning with inference-level memory by perturbing parameters in place and regenerating the same noise direction from a random seed. Given parameters $\boldsymbol{\theta} \in \mathbb{R}^{D}$, mini-batch $\mathcal{B}$, perturbation scale $\sigma$, and seed-generated direction $\mat{z}(s)$, it evaluates
\begin{equation}
\widehat{\mat{g}}(\boldsymbol{\theta};\mathcal{B})
=
\frac{L(\boldsymbol{\theta} + \sigma \mat{z};\mathcal{B}) - L(\boldsymbol{\theta} - \sigma \mat{z};\mathcal{B})}{2\sigma}\, \mat{z},
\end{equation}
or equivalently the projected scalar
\begin{equation}
\begin{aligned}
g_{\mathrm{proj}} &= \frac{\ell^{+}-\ell^{-}}{2\sigma}, \\
\ell^{+} &= L(\boldsymbol{\theta}+\sigma \mat{z}(s);\mathcal{B}), \\
\ell^{-} &= L(\boldsymbol{\theta}-\sigma \mat{z}(s);\mathcal{B}),
\end{aligned}
\end{equation}
followed by the in-place update $\boldsymbol{\theta} \leftarrow \boldsymbol{\theta} - \eta\, g_{\mathrm{proj}}\, \mat{z}(s)$. This design avoids storing activations or full perturbation tensors, but because each step uses only one direction, reducing estimator variance by averaging more directions increases forward cost linearly.

\begin{table*}[!t]
\centering
\small
\setlength{\tabcolsep}{3pt}
\begin{tabular}{p{2.6cm}p{2.2cm}p{3.4cm}p{3.2cm}cp{1.9cm}}
\toprule
\textbf{Method} & \textbf{Fwd passes/step} & \textbf{Extra memory vs.\ inference} & \textbf{Variance reduction} & \textbf{BP-free} & \textbf{Update space} \\
\midrule
MeZO~\citep{malladi2023fine}              & 2                     & none (inference level)$^*$         & ---                                  & \checkmark & Full param \\
HiZOO~\citep{zhao2024hizoo}$^\dagger$    & 3                     & fp32 diagonal Hessian              & Curvature preconditioning            & \checkmark & Full param \\
MeZO-SVRG~\citep{gautam2024mezosvrg}$^\ddagger$ & $2{+}$ periodic sweep & reference copy + full-batch est. & SVRG control variates             & \checkmark & Full param \\
FZOO~\citep{dang2025fzoo}$^\S$          & $N{+}1$ (one-sided)   & activations from $N$ parallel fwds & $N$ indep.\ directions               & \checkmark & Full param \\
Sparse-MeZO~\citep{liu2024sparse}        & 2                     & mask + mutation                    & Reduced update dimension             & \checkmark & Sparse subset \\
LOZO~\citep{chen2025enhancing} & 2 & inherits MeZO mutation             & Low-rank subspace constraint         & \checkmark & Low-rank \\
LoRA + Adam~\citep{hu2022lora}           & 2 (fwd+bwd)           & full-net activations + Adam state on adapter & ---                                  & $\times$   & Low-rank adapter \\
\midrule
\textbf{GRZO (ours)}                     & \textbf{2}            & none (inference level)$^{**}$ & \textbf{$B$ pseudo-indep.\ directions per step} & \checkmark & Full param \\
\bottomrule
\end{tabular}
\caption{Comparison of representative ZO fine-tuning methods, listing the qualitative source of each method's extra memory; quantitative measurements appear in Appendix~\ref{app:profile_breakdown}. $^*$We profile MeZO with a fully in-place perturb--restore--update implementation, which reaches the inference floor.
$^\dagger$HiZOO uses one extra forward per step and stores an fp32 diagonal Hessian.
$^\ddagger$MeZO-SVRG additionally performs periodic full-dataset sweeps.
$^\S$Parallel FZOO runs $N$ perturbed forwards concurrently; a sequential variant trades activation memory for $N{\times}$ wall-clock.
$^{**}$GRZO applies perturbations via forward pre-hooks, so at most one layer's perturbation is live at a time; it reaches the same inference floor while providing $B$ directions rather than one.}
\label{tab:related_work}
\end{table*}

\subsection{Flipout for Pseudo-Independent Per-Example Perturbations}

Flipout \citep{wen2018flipout} addresses the inefficiency of sharing one weight perturbation across an entire mini-batch. For a linear layer with weight matrix $\mat{W} \in \mathbb{R}^{d_{\text{out}} \times d_{\text{in}}}$, it samples a shared base perturbation $\mat{U}$ and constructs an effective perturbation for example $i$ as
\begin{equation}
\Delta \mat{W}_i \;=\; \mat{U} \odot (\mat{r}_i \mat{s}_i^\top),
\end{equation}
where $\mat{r}_i \in \{\pm 1\}^{d_{\text{out}}}$ and $\mat{s}_i \in \{\pm 1\}^{d_{\text{in}}}$ are independent sign vectors. Stacking a mini-batch of activations into $X$, and the sign vectors into matrices $R$ and $S$, yields the vectorized form
\begin{equation}
Y
=\phi\!\left(XW \;+\; \big((X \odot S)\mat{U}\big)\odot R\right),
\label{eq:flipout_vectorized}
\end{equation}
which avoids materializing $B$ separately perturbed weight matrices while still producing pseudo-independent example-level perturbations. GRZO uses this construction to obtain $B$ perturbation-induced loss signals within the same two-forward-pass budget used by MeZO.

\section{Comparison of ZO Fine-Tuning Methods}
\label{app:related_work_table}

Table~\ref{tab:related_work} compares representative zeroth-order fine-tuning methods across four dimensions: per-step forward-pass count, extra memory relative to MeZO, variance reduction mechanism, and whether backpropagation is required.

\section{Detailed Experimental Settings}
\label{app:exp_settings}

\paragraph{Hardware.}
All experiments are conducted on servers equipped with 8$\times$ NVIDIA A100 (40\,GB) or 8$\times$ NVIDIA A6000 (48\,GB) GPUs. Each Llama3-8B or OPT-13B fine-tuning run on a single task takes approximately 4--8 hours of wall-clock time on one such node, and longer for the quantized variants.

\paragraph{Aggregation.}
All accuracy and loss numbers reported in this paper are means across 3 runs (different random seeds) per task-method configuration.

\paragraph{Implementation.}
All experiments use PyTorch with the Hugging Face \texttt{transformers} library and standard model checkpoints (\texttt{roberta-large}, \texttt{meta-llama/Llama-3-8B}, \texttt{facebook/opt-13b}). LoRA adapters are implemented via the \texttt{peft} library; AdamW is from \texttt{torch.optim}. F1 scorers for SQuAD and DROP follow the official scripts from the respective dataset releases.

\paragraph{RoBERTa-large.}
We follow the experimental protocol of \citet{malladi2023fine} exactly, including data sampling, evaluation splits, and prompt templates. Each task uses $k{=}512$ labeled examples per class.

\paragraph{Llama3-8B and OPT-13B.}
Each task is trained on 1{,}000 examples (200 for CB, 350 for COPA), with a held-out development set of 500 examples used for learning rate selection and a test set of up to 1{,}000 examples. We train for 20{,}000 steps with a linear warmup over 500 steps and a constant learning rate thereafter. The batch size is 16 and we use FP16 precision. The perturbation scale is $\sigma{=}10^{-3}$. Learning rates for GRZO are selected from $\{1\text{e-}7,\,3\text{e-}7,\,5\text{e-}7\}$ based on development set accuracy; baseline settings follow their respective sources~\citep{malladi2023fine,dang2025fzoo}.

\paragraph{Perturbed Parameters.}
We apply full-parameter zeroth-order fine-tuning: all learnable parameters are perturbed, including linear projection weights (sign-factorized), embedding matrices (sparse row-wise perturbation indexed by active tokens), and normalization layer parameters (LayerNorm/RMSNorm). This matches the full-parameter MeZO setting.

\paragraph{Prompts and Task Formulation.}
We adopt the same prompt templates as \citet{malladi2023fine}. For multiple-choice tasks (CB, COPA, WiC, BoolQ, MultiRC), inference uses candidate log-likelihood scoring: each candidate is appended to the prompt and the candidate with the highest mean per-token log-likelihood is selected. During training, we apply teacher forcing on the correct candidate only, computing the loss solely on candidate tokens while excluding prompt tokens.

\paragraph{Hyperparameters.}
Tables~\ref{tab:hparams-roberta} and~\ref{tab:hparams-llm} summarize the hyperparameter settings for all methods. For all methods, learning rates are selected by grid search on the development set; the best value per task is reported. Adam and LoRA use AdamW with $\beta_1{=}0.9$, $\beta_2{=}0.999$.

\begin{table*}[!t]
\centering
\small
\caption{Hyperparameter settings for RoBERTa-large ($k{=}512$).}
\label{tab:hparams-roberta}
\setlength{\tabcolsep}{8pt}
\renewcommand{\arraystretch}{1.15}
\begin{tabular}{l l c}
\toprule
\textbf{Method} & \textbf{Hyperparameter} & \textbf{Value} \\
\midrule
MeZO
& Batch size        & 64 \\
& Learning rate     & $\{1\mathrm{e}{-7},\,5\mathrm{e}{-7},\,1\mathrm{e}{-6}\}$ \\
& Perturbation $\sigma$ & $1\mathrm{e}{-3}$ \\
& Weight decay      & 0 \\
\midrule
FZOO
& Batch size        & 64 \\
& Learning rate     & $\{1\mathrm{e}{-5},\,1\mathrm{e}{-4},\,5\mathrm{e}{-4}\}$ \\
& Perturbation $\sigma$ & $1\mathrm{e}{-3}$ \\
& Weight decay      & 0 \\
\midrule
GRZO (Ours)
& Batch size        & 64 \\
& Learning rate     & $\{1\mathrm{e}{-6},\,1\mathrm{e}{-5},\,5\mathrm{e}{-5}\}$ \\
& Perturbation $\sigma$ & $1\mathrm{e}{-3}$ \\
& Weight decay      & 0 \\
\bottomrule
\end{tabular}
\end{table*}

Adam (FO) and LoRA (FO) hyperparameters follow \citet{malladi2023fine} and prior work.

\begin{table*}[!t]
\centering
\small
\caption{Hyperparameter settings for Llama3-8B and OPT-13B.}
\label{tab:hparams-llm}
\setlength{\tabcolsep}{6pt}
\renewcommand{\arraystretch}{1.15}
\begin{tabular}{l l c c}
\toprule
\textbf{Method} & \textbf{Hyperparameter} & \textbf{Llama3-8B} & \textbf{OPT-13B} \\
\midrule
MeZO
& Batch size        & 16 & 16 \\
& Learning rate     & $\{1\mathrm{e}{-7},\,1\mathrm{e}{-6},\,1\mathrm{e}{-5}\}$ & $\{1\mathrm{e}{-7},\,1\mathrm{e}{-6},\,1\mathrm{e}{-5}\}$ \\
& Perturbation $\sigma$ & $1\mathrm{e}{-3}$ & $1\mathrm{e}{-3}$ \\
& Weight decay      & 0 & 0 \\
& Training steps    & 20{,}000 & 20{,}000 \\
\midrule
FZOO
& Batch size        & 16 & 16 \\
& Learning rate     & $\{1\mathrm{e}{-5},\,5\mathrm{e}{-5},\,1\mathrm{e}{-4}\}$ & $\{1\mathrm{e}{-5},\,5\mathrm{e}{-5},\,1\mathrm{e}{-4}\}$ \\
& Perturbation $\sigma$ & $1\mathrm{e}{-3}$ & $1\mathrm{e}{-3}$ \\
& Weight decay      & 0 & 0 \\
& Training steps    & 20{,}000 & 20{,}000 \\
\midrule
GRZO (Ours)
& Batch size        & 16 & 16 \\
& Learning rate     & $\{1\mathrm{e}{-7},\,3\mathrm{e}{-7},\,5\mathrm{e}{-7}\}$ & $\{1\mathrm{e}{-7},\,3\mathrm{e}{-7},\,5\mathrm{e}{-7}\}$ \\
& Perturbation $\sigma$ & $1\mathrm{e}{-3}$ & $1\mathrm{e}{-3}$ \\
& Weight decay      & 0 & 0 \\
& Training steps    & 20{,}000 & 20{,}000 \\
\bottomrule
\end{tabular}
\end{table*}

\section{Detailed Time and Memory Breakdown}
\label{app:profile_breakdown}

This appendix tabulates the per-step wall-clock decomposition and peak GPU memory measurements that underlie Section~\ref{sec:memory-analysis} and Figure~\ref{fig:profile}. All numbers are from the same production profile: Llama-3-8B (fp16), RTE, batch size $B{=}16$, 4$\times$A100-40GB, mean of 80 samples ($20$ steps $\times$ 4 ranks). In Table~\ref{tab:profile-time}, \textbf{Forward} sums the per-step forward passes (two standard for MeZO-family; two fused-perturbation for GRZO-family); \textbf{Perturb/Setup} covers perturbation handling (three in-place operations for MeZO-family; per-Linear sign-vector and base-noise setup for GRZO-family); \textbf{Update} is the weight-update step; \textbf{Other} is the residual (loss reduce + HF Trainer/DDP overhead).

\begin{table*}[!t]
\centering
\small
\caption{Per-step wall-clock breakdown (ms) across all profiled methods. \textbf{Total}s are production measurements; column definitions are given in the text above.}
\label{tab:profile-time}
\setlength{\tabcolsep}{5pt}
\renewcommand{\arraystretch}{1.12}
\begin{tabular}{@{}lrrrrr@{}}
\toprule
\rowcolor{tableheader}
\textbf{Method} & \textbf{Forward} & \textbf{Perturb/Setup} & \textbf{Update} & \textbf{Other} & \textbf{Total} \\
\midrule
\multicolumn{6}{@{}l}{\textit{MeZO family (in-place perturbation, two forward passes)}} \\
MeZO         & 243 & 318  & 229 & 0.1  & 790  \\
LOZO         & 243 & 421  & 268 & 0.1  & 932  \\
Sparse-MeZO  & 216 & 520  & 292 & 0.1  & 1028 \\
QuZO         & 243 & 1203 & 538 & 0.1  & 1984 \\
\midrule
\multicolumn{6}{@{}l}{\textit{GRZO family (fused per-example perturbation, two forward passes)}}\\
\rowcolor{grzopink}
GRZO         & 712 & 50  & 82 & 129  & 973  \\
\rowcolor{grzopinkdeep}
LO-GRZO    & 713 & 61  & 82 & 134  & 990  \\
\rowcolor{grzopinkdeep}
Sparse-GRZO  & 702 & 170 & 82 & 235  & 1189 \\
\rowcolor{grzopinkdeep}
Qu-GRZO    & 709 & 208 & 82 & 1567 & 2566 \\
\bottomrule
\end{tabular}
\end{table*}

\begin{table*}[!t]
\centering
\small
\caption{Peak GPU memory per step (GB), same setup as Table~\ref{tab:profile-time}. The forward-only inference peak is $17.74$~GB; \textbf{Over floor} = peak $-$ inference floor, in GB and as a percentage. MeZO-family methods are profiled with fully in-place implementations.}
\label{tab:profile-memory}
\setlength{\tabcolsep}{6pt}
\renewcommand{\arraystretch}{1.12}
\begin{tabular}{@{}lrrr@{}}
\toprule
\rowcolor{tableheader}
\textbf{Method} & \textbf{Peak (GB)} & \multicolumn{2}{c}{\textbf{Over floor}} \\
\cmidrule(l){3-4}
\rowcolor{tableheader}
 &  & \textbf{GB} & \textbf{\%} \\
\midrule
Forward-only inference & 17.74 & --- & --- \\
\midrule
\multicolumn{4}{@{}l}{\textit{MeZO family}} \\
MeZO         & 17.74 & $+0.00$ & $0.0$ \\
LOZO         & 17.78 & $+0.04$ & $0.2$ \\
Sparse-MeZO  & 24.64 & $+6.90$ & $38.9$ \\
QuZO         & 20.66 & $+2.92$ & $16.5$ \\
\midrule
\multicolumn{4}{@{}l}{\textit{GRZO family}} \\
\rowcolor{grzopink}
GRZO         & 17.82 & $+0.08$ & $0.5$ \\
\rowcolor{grzopinkdeep}
LO-GRZO      & 17.86 & $+0.12$ & $0.7$ \\
\rowcolor{grzopinkdeep}
Sparse-GRZO  & 24.31 & $+6.57$ & $37.0$ \\
\rowcolor{grzopinkdeep}
Qu-GRZO      & 19.14 & $+1.40$ & $7.9$ \\
\bottomrule
\end{tabular}
\end{table*}

\begin{figure*}[!ht]
    \centering
    \includegraphics[width=\textwidth]{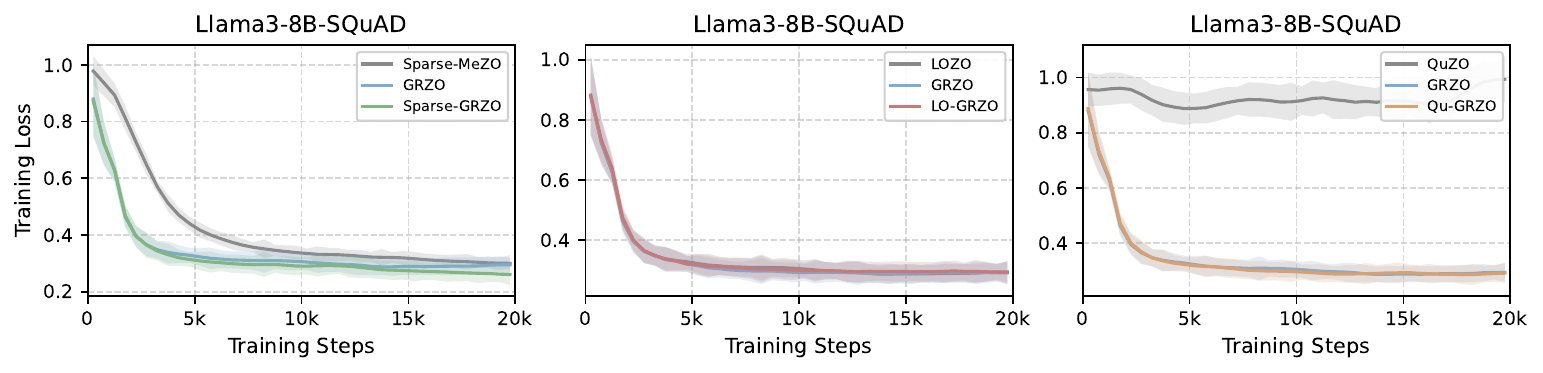}
    \hfill
    \includegraphics[width=\textwidth]{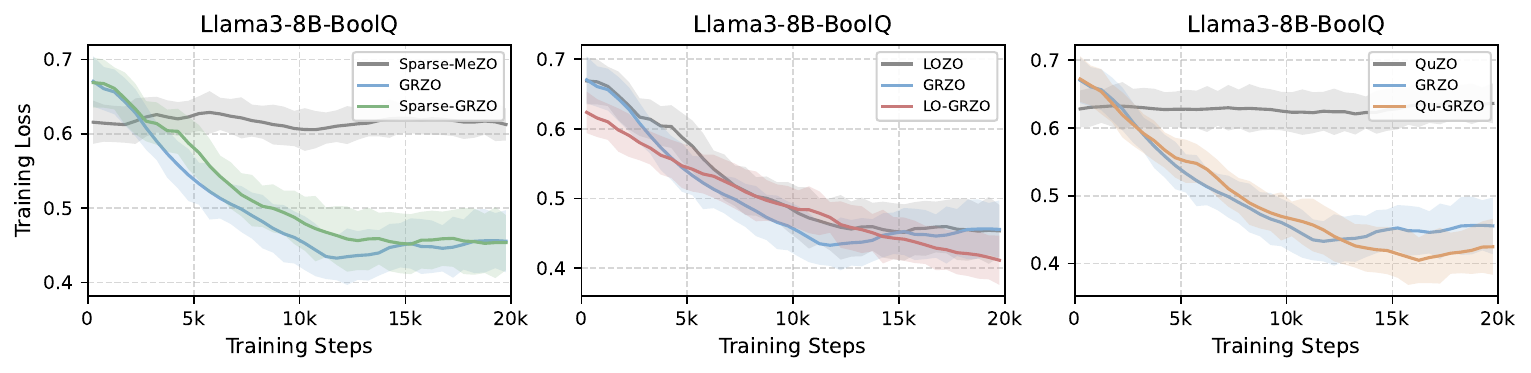}
    \caption{Per-variant GRZO+X vs MeZO+X training-loss curves on Llama3-8B. \textbf{Left}: SQuAD. \textbf{Right}: BoolQ.}
    \label{fig:per-variant-compare}
\end{figure*}

\paragraph{Observations.}
Two patterns dominate. First, MeZO-family in-place perturbation cost scales steeply with the variant's per-parameter work---LOZO ($1.32\times$), Sparse-MeZO ($1.64\times$), QuZO ($3.78\times$) over plain MeZO---because each variant traverses all $\sim$8B parameters during perturb, restore, and update. GRZO removes this scaling by representing each per-example perturbation as $B(d_{\text{out}}{+}d_{\text{in}})$ int-8 sign vectors per layer rather than a per-parameter modification; consequently the GRZO-family update step is uniformly $82$~ms across all four pairings, and combined-variant perturb/setup cost is $\leq 270$~ms. Second, on memory, what separates the rows of Table~\ref{tab:profile-memory} is the variant's own scratch storage, not the choice of ZO core: vanilla and LO-GRZO sit within $0.7\%$ of the inference floor, whereas Sparse-GRZO and Qu-GRZO carry the sparse mask and quantization buffers of their paired variant. Pairwise, GRZO costs $+0.08$~GB over MeZO on the vanilla and LOZO pairings and is $0.33$ and $1.52$~GB \emph{cheaper} on the Sparse and quantized pairings, where avoiding a mutated weight copy alongside the variant's own buffers pays off.

\paragraph{What the $B$ Directions Cost.}
Implemented at their respective optima, MeZO and GRZO both reach the inference floor, so the question Table~\ref{tab:profile-memory} answers is not which core is cheaper but what the extra $B{-}1$ directions cost: $+0.08$~GB, or $0.5\%$. Two mechanisms hold it there. Each layer's perturbed weight is constructed transiently inside a forward pre-hook and freed before the next layer runs, so at most one layer's perturbation is live at any moment regardless of depth; and the sign factorization compresses all per-example variation into $B(d_{\text{out}}{+}d_{\text{in}})$ int-8 sign vectors per layer, so the per-layer state of the $B$-direction estimator matches that of a single-direction one. Neither term scales with $B\times$ the parameter count, which is what a naive $B$-direction estimator would pay. The LOZO pairing reproduces this exactly: $+0.04$~GB for LOZO and $+0.12$~GB for LO-GRZO over the floor, again $+0.08$~GB for the extra directions. Sparse and quantized pairings are instead dominated by the variant's own per-parameter scratch storage, which is common to both families rather than attributable to the ZO core.

\paragraph{Caveat: FZOO Variant Ambiguity.}
FZOO~\citep{dang2025fzoo} admits both sequential and batched-parallel perturbation variants with substantially different wall-clock and memory profiles (the latter trading $O(N)$ activation memory for near-MeZO wall-clock). Reporting either single number would mischaracterize the method, and the choice is implementation-dependent rather than algorithmic. We therefore restrict the FZOO comparison to accuracy, which is independent of the parallel/sequential choice.

\paragraph{Caveat: QuZO and Qu-GRZO Profile Numbers.}
The peak-memory and per-step wall-clock numbers reported by \citet{zhou2025quzo} for QuZO (their Tables~4 and 5; Appendix~C) come from a Cutlass INT8 kernel that stores weights in packed int4/int8 format and dispatches GEMMs onto integer Tensor Core paths; this kernel is not part of the public release. The released code (\texttt{qft} mode) simulates low-bit fine-tuning by fake-quantizing weights inside each forward pass while retaining fp16 storage and fp16 compute, and our QuZO and Qu-GRZO implementations follow the same paradigm. Profiling either method on A100 in fp16 therefore captures fake-quant overhead, not the weight footprint and INT8-GEMM throughput that drive the deployment numbers of~\citet{zhou2025quzo}; the QuZO and Qu-GRZO rows in Tables~\ref{tab:profile-time}--\ref{tab:profile-memory} should be read as fake-quant simulation, not as comparable to QuZO's Table~5. We therefore omit QuZO and Qu-GRZO from the cross-variant memory and time claim in Section~\ref{sec:composing}. Their accuracy entries in Table~\ref{tab:additional-baselines} apply paper Algorithm~1 with identical W8 weight and perturbation quantization to both methods and remain apples-to-apples. Within the fake-quant regime on A100, Qu-GRZO is still empirically more memory-efficient than QuZO (though $29\%$ slower per step, the largest per-step premium of the four pairings); this memory advantage stems from forward-hook streaming rather than storage precision, and would persist or amplify on true low-bit hardware.

\section{GRZO-combined Variants Convergence Examples}
\label{app:per_variant_compare}

Figure~\ref{fig:per-variant-compare} shows per-task training-dynamics views of GRZO+X versus its paired MeZO variants and vanilla GRZO baseline on SQuAD and BoolQ (Llama3-8B).

\section{Batch Size Sensitivity}
\label{app:bs-ablation}

Figure~\ref{fig:bs-ablation} shows training loss curves for GRZO across batch sizes $B \in \{4, 8, 16, 32\}$ on two tasks: SST-2 (Llama3-8B) and COPA (OPT-13B).
The results corroborate the theoretical prediction in Section~\ref{sec:grzo_bs_sensitivity}: the group-relative normalizer requires a stable within-batch loss standard deviation $s$ to produce reliable advantage weights, and this stability breaks down at very small batch sizes.

\begin{figure}[!ht]
    \centering
    \includegraphics[width=\columnwidth]{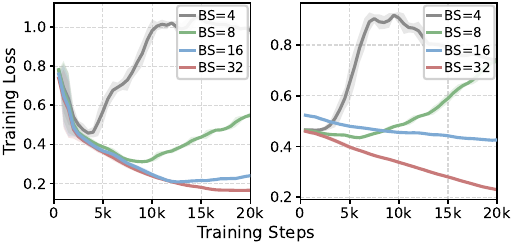}
    \caption{Batch size sensitivity of GRZO. \textbf{Left}: Llama3-8B/SST-2. \textbf{Right}: OPT-13B/COPA. $B{=}4$ diverges; $B{=}8$ unstable; $B{\geq}16$ converges stably.}
    \label{fig:bs-ablation}
\end{figure}

\section{Alignment with the True Gradient}
\label{app:cosine}

We compare each single ZO gradient estimate against the exact backpropagation gradient on one
pretrained Llama3-8B decoder block ($218$M parameters), over $3{,}000$ paired trials per batch size.
Table~\ref{tab:cosine} reports the cosine similarity between the estimate and the true gradient.
In $218$M dimensions the absolute cosines are small for both methods, but GRZO's grows with $B$
while MeZO's is flat, which is the behaviour predicted by \eqref{eq:beff}: MeZO draws one direction
regardless of $B$, whereas GRZO draws one per example. On the same block, the mean pairwise cosine
between per-example gradients is $c=0.83$, which by \eqref{eq:beff} gives $B_{\mathrm{eff}}\approx13$
at $B{=}16$.

\begin{table}[h]
\centering
\small
\caption{Cosine similarity to the exact backpropagation gradient on a Llama3-8B decoder block
($218$M parameters), in units of $10^{-5}$, mean $\pm$ standard error over $3{,}000$ paired trials per $B$.}
\label{tab:cosine}
\setlength{\tabcolsep}{5pt}
\renewcommand{\arraystretch}{1.12}
\begin{tabular}{@{}lccc@{}}
\toprule
\rowcolor{tableheader}
$B$ & \textbf{GRZO} & \textbf{MeZO} & \textbf{Ratio} \\
\midrule
1  & $5.280\pm0.072$  & $5.420\pm0.073$ & $0.97$ \\
4  & $11.87\pm0.091$  & $5.330\pm0.074$ & $2.23$ \\
8  & $17.21\pm0.094$  & $5.329\pm0.075$ & $3.23$ \\
16 & $24.51\pm0.096$  & $5.346\pm0.075$ & $4.58$ \\
32 & $34.80\pm0.099$  & $5.339\pm0.075$ & $6.52$ \\
\bottomrule
\end{tabular}
\end{table}

\section{Effect of Quantization on the Loss-Difference Scale}
\label{app:quant_scale}

To separate the effect of quantization on the loss differences from its effect on the update, we run
a controlled MLP probe ($B{=}16$, $20$k trials) with and without W4 weight quantization.
Quantization raises the dispersion of the per-example loss differences from
$\mathrm{std}(\delta)=0.0119$ to $0.0493$, a factor of $4.15$.
Table~\ref{tab:quant-scale} reports how this propagates to the mean update norm for the
MeZO estimator, the unnormalized canonical estimator, and GRZO.
MeZO and the canonical estimator scale by $4.13\times$ and $4.15\times$ respectively, matching the
dispersion increase; GRZO changes by $1.00\times$, since the group-relative weights
$a_i=\delta_i/(s+\epsilon)$ divide by a within-batch standard deviation that is inflated by the same
factor.

\begin{table}[h]
\centering
\small
\caption{Mean update norm with and without W4 weight quantization, on a controlled MLP probe
($B{=}16$, $20$k trials).}
\label{tab:quant-scale}
\setlength{\tabcolsep}{6pt}
\renewcommand{\arraystretch}{1.12}
\begin{tabular}{@{}lccc@{}}
\toprule
\rowcolor{tableheader}
\textbf{Estimator} & \textbf{No quant.} & \textbf{W4} & \textbf{Change} \\
\midrule
MeZO                 & $66.6$  & $274.8$ & $4.13\times$ \\
Canonical (unnorm.)  & $75.6$  & $313.9$ & $4.15\times$ \\
\rowcolor{grzopink}
GRZO                 & $13.42$ & $13.44$ & $1.00\times$ \\
\bottomrule
\end{tabular}
\end{table}

\onecolumn
\section{Unbiasedness and Smoothing Bias of GRZO}
\label{app:grzo_unbiased}

\subsection{Estimator Definition and Scope}
\label{app:grzo_est_def}

\paragraph{Joint Parameterization.}
GRZO perturbs all trainable blocks at every step, and each loss difference $\delta_i$ is produced by perturbing them \emph{simultaneously}. The analysis is therefore carried out on the concatenated parameter vector rather than one block at a time; Appendix~\ref{app:blockwise_full_network} explains why a per-block treatment is unavailable. Let
\begin{equation}
\boldsymbol{\theta}=\bigl(\mathrm{vec}(\mat{W}^{(1)}),\dots,\mathrm{vec}(\mat{W}^{(L)})\bigr)\in\mathbb{R}^{D},
\end{equation}
with $D=\sum_{\ell=1}^{L}D_{\text{out}}^{(\ell)}D_{\text{in}}^{(\ell)}$ the total number of perturbed parameters, and for each example $i\in\{1,\dots,B\}$ let
\begin{equation}
\mat{z}_i=\bigl(\mat{z}_i^{(1)},\dots,\mat{z}_i^{(L)}\bigr),\quad
\mat{z}_i^{(\ell)}=\mathrm{vec}\bigl(\mat{U}^{(\ell)}\!\odot(\mat{r}_i^{(\ell)}(\mat{s}_i^{(\ell)})^{\!\top})\bigr),
\label{eq:def_flipout_pert}
\end{equation}
with an independent base $\mat{U}^{(\ell)}$ per block and independent Rademacher sign pairs $(\mat{r}_i^{(\ell)},\mat{s}_i^{(\ell)})$ per block and example. Writing $\mat{u}\in\mathbb{R}^{D}$ for the concatenated base entries and $\mat{v}_i\in\{\pm1\}^{D}$ for the concatenated sign pattern, \eqref{eq:def_flipout_pert} reads compactly as $\mat{z}_i=\mat{u}\odot\mat{v}_i$. Every result below uses only this form, so all statements hold verbatim for a single layer ($L{=}1$, $D=D_{\text{out}}D_{\text{in}}$).

The two-sided losses and their differences are
\begin{equation}
\ell_i^{\pm} := \ell(\boldsymbol{\theta} \pm \sigma \mat{z}_i;\xi_i), \qquad \delta_i := \ell_i^{+}-\ell_i^{-}.
\label{eq:def_two_sided_loss}
\end{equation}

\paragraph{Canonical Two-Sided ZO Estimator.}
The standard (canonical) two-sided estimator is
\begin{equation}
\widehat{\mat{g}}_{\text{can}}(\boldsymbol{\theta})
:= \frac{1}{2\sigma B}\sum_{i=1}^{B} \delta_i\, \mat{z}_i .
\label{eq:g_canonical}
\end{equation}
This serves as the theoretical baseline from which GRZO departs.

\paragraph{GRZO Estimator (Group-Relative Normalization).}
GRZO computes the within-batch standard deviation
\begin{equation}
s = \sqrt{\frac{1}{B}\sum_{i=1}^{B}(\delta_i-\bar\delta)^2}, \qquad \bar\delta = \frac{1}{B}\sum_{i=1}^{B}\delta_i,
\label{eq:grzo_batch_stats}
\end{equation}
defines group-relative weights $a_i := \delta_i\,/\,(s+\epsilon)$, and forms
\begin{equation}
\widehat{\mat{g}}_{\text{GRZO}}(\boldsymbol{\theta})
:= \frac{1}{2\sigma B}\sum_{i=1}^{B} a_i\, \mat{z}_i.
\label{eq:g_grzo}
\end{equation}
The scaling by $1/(s+\epsilon)$ makes the update invariant to the magnitude of $\{\delta_i\}$, analogous to advantage normalization in GRPO~\citep{shao2024deepseekmath}. Because two-sided finite differences satisfy $\mathbb{E}[\delta_i]=0$ by symmetry of $\mat{z}_i \leftrightarrow -\mat{z}_i$, no explicit mean baseline is subtracted from the numerator.

\subsection{Marginal Distribution of Flipout Perturbations}
\label{app:flipout_marginal}

We only need the \emph{marginal} distribution of $\mat{z}_i$ for unbiasedness.
Let $\mat{U}$ have i.i.d.\ entries $\mat{U}_{jk}$ and let $\mat{r}_i\in\{\pm 1\}^{D_{\text{out}}}$,
$\mat{s}_i\in\{\pm 1\}^{D_{\text{in}}}$ be independent Rademacher vectors.
Then each coordinate of $\mat{z}_i$ is of the form
\begin{equation}
(\mat{z}_i)_{(j,k)} = \mat{U}_{jk}\cdot (\mat{r}_i)_j\cdot (\mat{s}_i)_k,
\label{eq:coord_signflip}
\end{equation}
i.e., a sign flip of $\mat{U}_{jk}$.

\paragraph{Gaussian Base Noise (Exact).}
If $\mat{U}_{jk}\sim\mathcal{N}(0,1)$, then $(\mat{z}_i)_{(j,k)} \sim \mathcal{N}(0,1)$ as well,
because multiplying a standard normal by an independent $\pm 1$ does not change its distribution.
Hence $\mat{z}_i \sim \mathcal{N}(0,\mat{I}_D)$ marginally.

\paragraph{Rademacher Base Noise (Isotropic, Symmetric).}
If $\mat{U}_{jk}\sim \mathrm{Rad}(\pm1)$, then $(\mat{z}_i)_{(j,k)}\sim \mathrm{Rad}(\pm1)$ marginally.
In either case, $\mat{z}_i$ is symmetric ($\mat{z}_i \stackrel{d}{=} -\mat{z}_i$) and isotropic
($\mathbb{E}[\mat{z}_i]=0$, $\mathbb{E}[\mat{z}_i \mat{z}_i^\top]=\mat{I}_D$).

\subsection{Unbiasedness w.r.t.\ a Smoothed Objective (Gaussian Case: Exact)}
\label{app:unbiased_gaussian_exact}

In the Gaussian case, define the Gaussian-smoothed population objective
\begin{equation}
F_\sigma(\boldsymbol{\theta}) := \mathbb{E}_{u\sim\mathcal{N}(0,\mat{I}_D)}\big[F(\boldsymbol{\theta}+\sigma \mat{u})\big].
\label{eq:def_gauss_smooth}
\end{equation}
By Stein's identity, we have
\begin{equation}
\nabla F_\sigma(\boldsymbol{\theta})
= \frac{1}{\sigma}\,\mathbb{E}_{u}\big[F(\boldsymbol{\theta}+\sigma \mat{u})\,\mat{u}\big].
\label{eq:stein_identity}
\end{equation}
Using symmetry of the Gaussian, $\mat{u}\stackrel{d}{=}-\mat{u}$, we obtain the antithetic form:
\begin{align}
\nabla F_\sigma(\boldsymbol{\theta})
&= \frac{1}{2\sigma}\,\mathbb{E}_{u}\Big[\big(F(\boldsymbol{\theta}+\sigma \mat{u})-F(\boldsymbol{\theta}-\sigma \mat{u})\big)\,\mat{u}\Big].
\label{eq:stein_antithetic}
\end{align}

Now consider a single-example objective $f_i(\boldsymbol{\theta}):=\mathbb{E}_{\xi_i}[\ell(\boldsymbol{\theta};\xi_i)]$.
If $\mat{z}_i\sim\mathcal{N}(0,\mat{I}_D)$ marginally (Appendix~\ref{app:flipout_marginal}),
then applying \eqref{eq:stein_antithetic} to $f_i$ yields
\begin{equation}
\mathbb{E}_{z_i}\left[\frac{f_i(\boldsymbol{\theta}+\sigma \mat{z}_i)-f_i(\boldsymbol{\theta}-\sigma \mat{z}_i)}{2\sigma}\,\mat{z}_i\right]
= \nabla f_{i,\sigma}(\boldsymbol{\theta}),
\label{eq:single_exact_unbiased}
\end{equation}
where $f_{i,\sigma}$ is the Gaussian smoothing of $f_i$.

Taking expectation over the mini-batch sampling and averaging over $i$,
linearity of expectation gives the unbiasedness of the canonical estimator:
\begin{equation}
\mathbb{E}\big[\widehat{\mat{g}}_{\text{can}}(\boldsymbol{\theta})\big]
= \nabla F_\sigma(\boldsymbol{\theta}).
\label{eq:exact_unbiased_gauss}
\end{equation}

\paragraph{GRZO as Self-Normalized Estimator.}
GRZO applies a positive self-normalization to the canonical estimator. Since $a_i = \delta_i/(s+\epsilon)$, the estimator factors as
\begin{equation}
\widehat{\mat{g}}_{\text{GRZO}}(\boldsymbol{\theta})
= \frac{1}{s+\epsilon}\cdot\frac{1}{2\sigma B}\sum_{i=1}^{B}\delta_i\,\mat{z}_i
= \frac{1}{s+\epsilon}\cdot\widehat{\mat{g}}_{\text{can}}(\boldsymbol{\theta}).
\label{eq:grzo_decomp}
\end{equation}
The scalar $1/(s+\epsilon)$ depends on $\{\mat{z}_i\}_{i=1}^B$ through $\{\delta_i\}$; it is therefore random and correlated with $\widehat{\mat{g}}_{\text{can}}$, and cannot be factored out of the expectation or absorbed into the learning rate by taking expectations. Equivalently, the GRZO update can be written as $\boldsymbol{\theta}_{t+1} = \boldsymbol{\theta}_t - \tilde{\eta}_t\,\widehat{\mat{g}}_{\text{can}}(\boldsymbol{\theta}_t)$ with a \emph{random} effective step size $\tilde{\eta}_t := \eta_t/(s_t+\epsilon)$, paralleling the FZOO~\citep{dang2025fzoo} convergence framing. Rather than assume that $s$ concentrates, we bound its finite-batch effect directly in Appendix~\ref{app:grzo_normalization}: with $c_\star=1/(s_\star+\epsilon)$, $s_\star^2=\mathbb{E}[s^2]$, Lemma~\ref{lem:s_concentration} and Corollary~\ref{cor:grzo_normalized} give
\begin{equation}
\mathbb{E}\big[\widehat{\mat{g}}_{\text{GRZO}}(\boldsymbol{\theta})\big]
\;=\; c_\star\bigl(1+O(B^{-1}{+}d_{\mathrm{eff}}^{-1})\bigr)\,\nabla F_\sigma(\boldsymbol{\theta}) + O(\sigma^2).
\label{eq:grzo_approx_unbiased}
\end{equation}
GRZO is thus direction-preserving up to an explicit finite-batch correction that vanishes as $B$ grows, and the convergence analysis (Appendix~\ref{app:grzo_convergence}) runs on the effective step size $\tilde\eta=\eta\,c_\star$ with this correction carried through the constants.

\subsection{Taylor Expansion View (General Symmetric Isotropic Directions)}
\label{app:taylor_general}

This subsection provides an alternative derivation via Taylor expansion, which applies to
any symmetric isotropic direction distribution (including Rademacher).

Assume $f(\boldsymbol{\theta})$ is three-times differentiable and its third derivative tensor is bounded:
\begin{equation}
\|\nabla^3 f(\boldsymbol{\theta})\|_{\mathrm{op}} \le \rho, \quad \forall \boldsymbol{\theta}.
\label{eq:assump_third_deriv}
\end{equation}
For a fixed direction $\mat{z}$,
a third-order Taylor expansion gives
\begin{align}
f(\boldsymbol{\theta}+\sigma \mat{z})
&= f(\boldsymbol{\theta}) + \sigma \langle \nabla f(\boldsymbol{\theta}), \mat{z}\rangle
+ \frac{\sigma^2}{2} \mat{z}^\top \nabla^2 f(\boldsymbol{\theta})\, \mat{z}
+ \frac{\sigma^3}{6}\nabla^3 f(\boldsymbol{\theta})[\mat{z},\mat{z},\mat{z}] + O(\sigma^4\|\mat{z}\|^4),
\\
f(\boldsymbol{\theta}-\sigma \mat{z})
&= f(\boldsymbol{\theta}) - \sigma \langle \nabla f(\boldsymbol{\theta}), \mat{z}\rangle
+ \frac{\sigma^2}{2} \mat{z}^\top \nabla^2 f(\boldsymbol{\theta})\, \mat{z}
- \frac{\sigma^3}{6}\nabla^3 f(\boldsymbol{\theta})[\mat{z},\mat{z},\mat{z}] + O(\sigma^4\|\mat{z}\|^4).
\end{align}
Subtracting,
\begin{equation}
f(\boldsymbol{\theta}+\sigma \mat{z})-f(\boldsymbol{\theta}-\sigma \mat{z})
= 2\sigma\langle \nabla f(\boldsymbol{\theta}), \mat{z}\rangle
+ \frac{\sigma^3}{3}\nabla^3 f(\boldsymbol{\theta})[\mat{z},\mat{z},\mat{z}]
+ O(\sigma^5\|\mat{z}\|^5).
\label{eq:taylor_diff}
\end{equation}
Plugging \eqref{eq:taylor_diff} into the two-sided estimator for a single direction,
\begin{align}
\frac{f(\boldsymbol{\theta}+\sigma \mat{z})-f(\boldsymbol{\theta}-\sigma \mat{z})}{2\sigma}\, \mat{z}
&= \langle \nabla f(\boldsymbol{\theta}), \mat{z}\rangle \mat{z}
+ \frac{\sigma^2}{6}\nabla^3 f(\boldsymbol{\theta})[\mat{z},\mat{z},\mat{z}]\, \mat{z}
+ O(\sigma^4\|\mat{z}\|^6).
\label{eq:taylor_est}
\end{align}
Taking expectation over a symmetric isotropic $\mat{z}$ with $\mathbb{E}[\mat{z}\mat{z}^\top]=\mat{I}$ yields
\begin{equation}
\mathbb{E}\big[\langle \nabla f(\boldsymbol{\theta}), \mat{z}\rangle \mat{z}\big]
= \mathbb{E}[\mat{z}\mat{z}^\top]\nabla f(\boldsymbol{\theta})
= \nabla f(\boldsymbol{\theta}).
\label{eq:main_term}
\end{equation}
For the remainder, symmetry implies cancellation of odd moments and
\eqref{eq:assump_third_deriv} implies
\begin{equation}
\left\|\mathbb{E}\left[\frac{\sigma^2}{6}\nabla^3 f(\boldsymbol{\theta})[\mat{z},\mat{z},\mat{z}]\, \mat{z}\right]\right\|
\le \frac{\sigma^2}{6}\rho\,\mathbb{E}\|\mat{z}\|^4
= O(\sigma^2).
\label{eq:bias_order_sigma2}
\end{equation}
Therefore,
\begin{equation}
\mathbb{E}\left[\frac{f(\boldsymbol{\theta}+\sigma \mat{z})-f(\boldsymbol{\theta}-\sigma \mat{z})}{2\sigma}\, \mat{z}\right]
= \nabla f(\boldsymbol{\theta}) + O(\sigma^2).
\label{eq:approx_unbiased}
\end{equation}
Averaging over examples yields $\mathbb{E}[\widehat{\mat{g}}_{\text{can}}(\boldsymbol{\theta})]
= \nabla F(\boldsymbol{\theta}) + O(\sigma^2)$.

\subsection{Smoothing Bias: Smoothed vs. Original Gradient}
\label{app:smoothing_bias}

We bound $\|\nabla F_\sigma(\boldsymbol{\theta})-\nabla F(\boldsymbol{\theta})\|$.
Assume $\nabla^2 F$ is $\rho$-Lipschitz (equivalently, \eqref{eq:assump_third_deriv} holds for $F$).
Then for any $\mat{u}$,
\begin{equation}
\nabla F(\boldsymbol{\theta}+\sigma \mat{u})
= \nabla F(\boldsymbol{\theta}) + \nabla^2 F(\boldsymbol{\theta})\,\sigma \mat{u} + R(\boldsymbol{\theta},\mat{u}),
\qquad
\|R(\boldsymbol{\theta},\mat{u})\|\le \frac{\rho}{2}\sigma^2\|\mat{u}\|^2.
\label{eq:grad_taylor_remainder}
\end{equation}
Taking expectation over $\mat{u}\sim\mathcal{N}(0,\mat{I}_D)$ (or any symmetric isotropic distribution),
the linear term vanishes since $\mathbb{E}[\mat{u}]=0$, giving
\begin{align}
\|\nabla F_\sigma(\boldsymbol{\theta})-\nabla F(\boldsymbol{\theta})\|
&= \left\|\mathbb{E}_{u}[\nabla F(\boldsymbol{\theta}+\sigma \mat{u})]-\nabla F(\boldsymbol{\theta})\right\|
= \left\|\mathbb{E}_{u}[R(\boldsymbol{\theta},\mat{u})]\right\|
\le \frac{\rho}{2}\sigma^2\,\mathbb{E}\|\mat{u}\|^2.
\end{align}
For Gaussian $\mat{u}\sim\mathcal{N}(0,\mat{I}_D)$, $\mathbb{E}\|\mat{u}\|^2=D$, hence
\begin{equation}
\|\nabla F_\sigma(\boldsymbol{\theta})-\nabla F(\boldsymbol{\theta})\|
\le \frac{\rho}{2}\sigma^2 D.
\label{eq:smoothing_bias_bound}
\end{equation}

\paragraph{Remark (Weaker Assumption).}
If we only assume $F$ is $L$-smooth (i.e., $\nabla F$ is $L$-Lipschitz),
then one can obtain the weaker bound
$\|\nabla F_\sigma(\boldsymbol{\theta})-\nabla F(\boldsymbol{\theta})\|
\le L\sigma\,\mathbb{E}\|\mat{u}\| = O(\sigma\sqrt{d})$ for Gaussian $\mat{u}$.

\subsection{Second Moment and Variance Bound of GRZO}
\label{app:grzo_second_moment}

\paragraph{Setup.}
Let $\boldsymbol{\theta}\in\mathbb{R}^{D}$ be the concatenated parameter vector of Appendix~\ref{app:grzo_est_def}.
For each flattened example $i\in\{1,\dots,B\}$, define the Flipout perturbation
$\mat{z}_i=\mathrm{vec}(\Delta \mat{W}_i)$ with $\Delta \mat{W}_i=\mat{U}\odot(\mat{r}_i\mat{s}_i^\top)$ as in
\eqref{eq:def_flipout_pert}. Consider the two-sided loss difference
$\delta_i=\ell(\boldsymbol{\theta}+\sigma \mat{z}_i;\xi_i)-\ell(\boldsymbol{\theta}-\sigma \mat{z}_i;\xi_i)$ and the canonical estimator
\begin{equation}
\widehat g(\boldsymbol{\theta})
:= \widehat g_{\text{can}}(\boldsymbol{\theta})
= \frac{1}{2\sigma B}\sum_{i=1}^{B}\delta_i\, \mat{z}_i.
\label{eq:app_g_def}
\end{equation}

\paragraph{Assumptions for This Subsection.}
We assume:
(i) $\ell(\cdot;\xi)$ is three-times differentiable and $\|\nabla^3 \ell(\boldsymbol{\theta};\xi)\|_{\mathrm{op}}\le \rho$;
(ii) the per-example gradients are bounded in second moment:
$\mathbb{E}\|\nabla \ell(\boldsymbol{\theta};\xi)\|^2 \le \|\nabla F(\boldsymbol{\theta})\|^2+\nu^2$;
(iii) the base noise has i.i.d.\ entries with $\mathbb{E}[\mat{U}_{jk}^2]=1$ and fourth moment
$m_4:=\mathbb{E}[\mat{U}_{jk}^4]$ (Gaussian: $m_4=3$, Rademacher: $m_4=1$).

\subsubsection{Key Conditional Independence}
\label{app:flipout_decorrelation}

Let $\mat{u}=\mathrm{vec}(\mat{U})\in\mathbb{R}^{D}$ and define the sign vector
$\mat{v}_i=\mathrm{vec}(\mat{r}_i\mat{s}_i^\top)\in\{\pm 1\}^d$. Then
\begin{equation}
\mat{z}_i = \mat{u} \odot \mat{v}_i .
\label{eq:zi_u_vi}
\end{equation}
Since $\{(\mat{r}_i,\mat{s}_i)\}_{i=1}^B$ are independent across $i$, the sign vectors $\{\mat{v}_i\}_{i=1}^B$
are independent across $i$. Therefore:

\begin{lemma}[Conditional independence and moments]
\label{lem:cond_indep}
Conditioned on $\mat{U}$ (equivalently, on $\mat{u}$), the perturbations $\{\mat{z}_i\}_{i=1}^B$ are independent.
Moreover,
\[
\mathbb{E}[\mat{z}_i\mid \mat{U}]=0,\qquad
\mathbb{E}[\mat{z}_i \mat{z}_i^\top \mid \mat{U}]=\mathrm{diag}(\mat{u}^2)=:\mat{\Lambda}_{\mat{U}},
\qquad
\mathbb{E}[\mat{z}_i \mat{z}_j^\top\mid \mat{U}]=0 \ (i\neq j).
\]
Unconditionally, $\mathbb{E}[\mat{z}_i \mat{z}_i^\top]=\mat{I}_D$.
\end{lemma}

\noindent
\emph{Proof.}
Given $\mat{U}$, each coordinate satisfies $(\mat{z}_i)_k=u_k (\mat{v}_i)_k$ with $\mathbb{E}[(\mat{v}_i)_k]=0$ and
$\mathbb{E}[(\mat{v}_i)_k(\mat{v}_i)_{k'}]=0$ for $k\neq k'$ (factorized Rademacher signs).
Independence across $i$ follows from independence of $\{\mat{v}_i\}$. $\square$

\subsubsection{Taylor Reduction from Perturbation Products to Quadratic Forms}
\label{app:taylor_to_quadratic}

Fix $\xi_i$ and denote $\mat{g}_i := \nabla_{\boldsymbol{\theta}} \ell(\boldsymbol{\theta};\xi_i)$.
By the third-order Taylor expansion (see Appendix~\ref{app:taylor_general}),
for each $i$ there exists a remainder $R_i$ such that
\begin{equation}
\delta_i
= 2\sigma \langle \mat{g}_i, \mat{z}_i\rangle + R_i,
\qquad
|R_i|\le \frac{\rho}{3}\sigma^3 \|\mat{z}_i\|^3.
\label{eq:delta_taylor}
\end{equation}
Plugging into $\delta_i \mat{z}_i$ gives the decomposition
\begin{equation}
\delta_i \mat{z}_i
= 2\sigma\, \underbrace{\big(\langle \mat{g}_i,\mat{z}_i\rangle \mat{z}_i\big)}_{:=\mat{q}_i}
\;+\; \underbrace{R_i \mat{z}_i}_{:=\mat{e}_i}.
\label{eq:delta_z_decomp}
\end{equation}
Hence
\begin{equation}
\widehat g(\boldsymbol{\theta})
= \frac{1}{B}\sum_{i=1}^{B}\mat{q}_i \;+\; \frac{1}{2\sigma B}\sum_{i=1}^{B}\mat{e}_i.
\label{eq:g_xi_eps}
\end{equation}

\subsubsection{Expanding the Estimator Second Moment: Diagonal and Cross Terms}
\label{app:expand_second_moment}

\paragraph{Conditioning convention.}
Throughout this subsubsection and the next, we compute second-moment expressions treating the per-example gradients $\mat{g}_i = \nabla_{\boldsymbol{\theta}} \ell(\boldsymbol{\theta};\xi_i)$ as held fixed, i.e.\ all expectations are conditioned implicitly on $\{\xi_i\}_{i=1}^{B}$. The outer data expectation $\mathbb{E}_\xi$ is applied at the final substitution step \eqref{eq:final_variance_bound} via the standard bounded-variance bound $\sum_i \mathbb{E}_{\xi_i}\|\mat{g}_i\|^2 \le B(\|\nabla F(\boldsymbol{\theta})\|^2+\nu^2)$ from assumption~(ii). Inner conditional expectations $\mathbb{E}[\cdot\mid\mat{U}]$ average over the per-example sign vectors $\{\mat{v}_i\}$ only; the outer expectations $\mathbb{E}[\cdot]$ average additionally over $\mat{U}$.

We first analyze the leading term $\frac{1}{B}\sum_i \mat{q}_i$.
Using $\|\sum_i a_i\|^2 = \sum_i \|a_i\|^2 + \sum_{i\neq j} a_i^\top a_j$,
\begin{align}
\mathbb{E}\left\|\frac{1}{B}\sum_{i=1}^{B}\mat{q}_i\right\|^2
&= \frac{1}{B^2}\sum_{i=1}^{B}\mathbb{E}\|\mat{q}_i\|^2
\;+\;\frac{1}{B^2}\sum_{i\neq j}\mathbb{E}\big[\mat{q}_i^\top \mat{q}_j\big].
\label{eq:diag_cross}
\end{align}

\paragraph{Diagonal Term.}
Condition on $\mat{U}$ and $\mat{q}_i$. Note that $\|\mat{z}_i\|^2=\|\mat{u}\|^2$ does not depend on $\mat{v}_i$.
Moreover,
$\mat{q}_i = \langle \mat{g}_i,\mat{z}_i\rangle \mat{z}_i$ implies
\[
\|\mat{q}_i\|^2 = (\langle \mat{g}_i,\mat{z}_i\rangle)^2 \|\mat{z}_i\|^2.
\]
Using Lemma~\ref{lem:cond_indep} and $\mathbb{E}[(\langle \mat{g}_i,\mat{z}_i\rangle)^2\mid \mat{U}]
= \mat{g}_i^\top \mat{\Lambda}_{\mat{U}} \mat{g}_i = \sum_{k=1}^{D} g_{i,k}^2 u_k^2$, we get
\begin{align}
\mathbb{E}\|\mat{q}_i\|^2
&= \mathbb{E}\Big[\|\mat{u}\|^2 \cdot \mat{g}_i^\top \mat{\Lambda}_{\mat{U}} \mat{g}_i\Big]
= \sum_{k=1}^{D} g_{i,k}^2 \, \mathbb{E}\big[\|\mat{u}\|^2 u_k^2\big].
\label{eq:diag_term_pre}
\end{align}
For i.i.d.\ coordinates with $\mathbb{E}[u_k^2]=1$ and $\mathbb{E}[u_k^4]=m_4$,
\begin{equation}
\mathbb{E}\big[\|\mat{u}\|^2 u_k^2\big]
= \mathbb{E}[u_k^4] + \sum_{\ell\neq k}\mathbb{E}[u_\ell^2]\mathbb{E}[u_k^2]
= m_4 + (D-1).
\label{eq:moment_identity}
\end{equation}
Thus
\begin{equation}
\mathbb{E}\|\mat{q}_i\|^2
= (D-1+m_4)\, \|\mat{g}_i\|^2.
\label{eq:diag_term_final}
\end{equation}

\paragraph{Cross Term (Second Moment).}
For $i\neq j$, conditioned on $\mat{U}$ the vectors $\mat{z}_i$ and $\mat{z}_j$ are independent
(Lemma~\ref{lem:cond_indep}), hence $\mat{q}_i$ and $\mat{q}_j$ are conditionally independent.
Therefore
\begin{equation}
\mathbb{E}[\mat{q}_i^\top \mat{q}_j\mid \mat{U}]
= \mathbb{E}[\mat{q}_i\mid \mat{U}]^\top \mathbb{E}[\mat{q}_j\mid \mat{U}].
\label{eq:cross_cond}
\end{equation}
Moreover, $\mathbb{E}[\mat{q}_i\mid \mat{U}]=\mathbb{E}[\langle \mat{g}_i,\mat{z}_i\rangle \mat{z}_i\mid \mat{U}]
= \mathbb{E}[\mat{z}_i \mat{z}_i^\top\mid \mat{U}] \mat{g}_i = \mat{\Lambda}_{\mat{U}}\mat{g}_i$.
Thus, per the conditioning convention above (gradients $\mat{g}_i$ held fixed in expectation),
\begin{equation}
\mathbb{E}[\mat{q}_i^\top \mat{q}_j]
= \mathbb{E}\big[\mat{g}_i^\top \mat{\Lambda}_{\mat{U}}^2 \mat{g}_j\big]
= \sum_{k=1}^{D} g_{i,k} g_{j,k}\, \mathbb{E}[u_k^4]
= m_4\, \langle \mat{g}_i, \mat{g}_j\rangle.
\label{eq:cross_term_final}
\end{equation}
Under i.i.d.\ data sampling, the subsequent outer data expectation gives $\mathbb{E}_\xi[\langle\mat{g}_i,\mat{g}_j\rangle]=\|\nabla F(\boldsymbol{\theta})\|^2$ for $i\neq j$. Centering removes the part of this contribution proportional to $\|\bar{\mat{g}}\|^2$, but---because the $B$ perturbations share one base $\mat{U}$---not all of it: the residual reappears as the shared-base term \eqref{eq:variance_base_term} of Subsubsection~\ref{app:variance_cross_vanish}.

\paragraph{Putting Diagonal + Cross Together.}
Combining \eqref{eq:diag_cross}, \eqref{eq:diag_term_final}, and \eqref{eq:cross_term_final},
we obtain
\begin{align}
\mathbb{E}\left\|\frac{1}{B}\sum_{i=1}^{B}\mat{q}_i\right\|^2
&= \frac{D-1+m_4}{B^2}\sum_{i=1}^{B}\|\mat{g}_i\|^2
\;+\;\frac{m_4}{B^2}\sum_{i\neq j}\langle \mat{g}_i,\mat{g}_j\rangle
\nonumber\\
&= \frac{D-1}{B^2}\sum_{i=1}^{B}\|\mat{g}_i\|^2
\;+\; \frac{m_4}{B^2}\left\|\sum_{i=1}^{B} \mat{g}_i\right\|^2.
\label{eq:second_moment_xi}
\end{align}

\subsubsection{Variance via the Law of Total Variance}
\label{app:variance_cross_vanish}

The quantity above is a second moment. We now bound the variance of
$\mat{Q}:=\frac{1}{B}\sum_i\mat{q}_i$ around its mean $\mathbb{E}[\mat{Q}]=\bar{\mat{g}}$,
where $\bar{\mat{g}}:=\frac{1}{B}\sum_i\mat{g}_i$. Conditioning on the shared base $\mat{U}$
and applying the law of total variance,
\begin{equation}
\mathrm{Var}(\mat{Q})
= \underbrace{\mathbb{E}\bigl[\mathrm{Var}(\mat{Q}\mid \mat{U})\bigr]}_{\text{(I) per-example noise}}
\;+\;\underbrace{\mathrm{Var}\bigl(\mathbb{E}[\mat{Q}\mid \mat{U}]\bigr)}_{\text{(II) shared-base noise}}.
\label{eq:total_variance}
\end{equation}
Both terms must be retained: conditioning on $\mat{U}$ removes the cross-example
covariance, but the fluctuation of $\mat{U}$ itself is common to all $B$ examples
and is therefore not reduced by per-example averaging.

\paragraph{Term (I): Per-Example Noise.}
Let $\bar{\mat{q}}_i := \mat{q}_i - \mathbb{E}[\mat{q}_i\mid \mat{U}]$. Conditioned on $\mat{U}$, the $\{\bar{\mat{q}}_i\}$ are
independent and zero-mean by Lemma~\ref{lem:cond_indep}, so for $i\neq j$,
\begin{equation}
\mathbb{E}\big[\bar{\mat{q}}_i^\top \bar{\mat{q}}_j \mid \mat{U}\big]=0,
\label{eq:cross_zero_centered}
\end{equation}
and hence
\begin{align}
\mathrm{Var}(\mat{Q}\mid \mat{U})
&= \frac{1}{B^2}\sum_{i=1}^{B}\mathbb{E}\bigl[\|\bar{\mat{q}}_i\|^2 \bigm| \mat{U}\bigr].
\label{eq:var_sum_diag_only}
\end{align}
Taking the outer expectation over $\mat{U}$ and using \eqref{eq:diag_term_final} together with
$\mathbb{E}\|\mathbb{E}[\mat{q}_i\mid\mat{U}]\|^2=\mathbb{E}\|\mat{\Lambda}_{\mat{U}}\mat{g}_i\|^2=m_4\|\mat{g}_i\|^2$,
\begin{align}
\mathbb{E}\|\bar{\mat{q}}_i\|^2
&= \mathbb{E}\|\mat{q}_i\|^2 - \mathbb{E}\|\mathbb{E}[\mat{q}_i\mid \mat{U}]\|^2
\nonumber\\
&= (D-1+m_4)\|\mat{g}_i\|^2 - m_4\|\mat{g}_i\|^2
= (D-1)\|\mat{g}_i\|^2,
\label{eq:var_xi}
\end{align}
so that
\begin{equation}
\text{(I)} \;=\; \frac{D-1}{B^2}\sum_{i=1}^{B}\|\mat{g}_i\|^2 .
\label{eq:variance_main_term}
\end{equation}
This is the term that Flipout decorrelation reduces by a factor $1/B$.

\paragraph{Term (II): Shared-Base Noise.}
By Lemma~\ref{lem:cond_indep}, $\mathbb{E}[\mat{q}_i\mid\mat{U}]=\mat{\Lambda}_{\mat{U}}\mat{g}_i$, so
$\mathbb{E}[\mat{Q}\mid\mat{U}]=\mat{\Lambda}_{\mat{U}}\bar{\mat{g}}$ with $\mathbb{E}[\mat{\Lambda}_{\mat{U}}]=\mat{I}_D$. Therefore
\begin{align}
\text{(II)}
&= \mathbb{E}\|\mat{\Lambda}_{\mat{U}}\bar{\mat{g}}\|^2 - \|\bar{\mat{g}}\|^2
= \sum_{k=1}^{D}\bar g_k^2\bigl(\mathbb{E}[u_k^4]-1\bigr)
\nonumber\\
&= (m_4-1)\,\|\bar{\mat{g}}\|^2 .
\label{eq:variance_base_term}
\end{align}
Term (II) carries no $1/B$ factor: it depends on the mini-batch only through the
mean gradient $\bar{\mat{g}}$.

\paragraph{Remark (Gaussian vs.\ Rademacher).}
For Rademacher base noise $m_4=1$, term (II) vanishes identically and
$\mathrm{Var}(\mat{Q})$ is exactly the $1/B$-scaled expression \eqref{eq:variance_main_term}.
For Gaussian base noise $m_4=3$ and term (II) equals $2\|\bar{\mat{g}}\|^2$, a floor that
persists as $B\to\infty$. The degenerate case $D=1$ separates the two terms unambiguously:
\eqref{eq:variance_main_term} is then identically zero, whereas for Gaussian base noise the
estimator reduces to $\bar g\,u^2$ with variance $2\bar g^2$, which is exactly
\eqref{eq:variance_base_term}. Retaining term (II) is thus necessary rather than merely
conservative.

\subsubsection{Remainder Control and Final Bound with Explicit Dimensions}
\label{app:final_bound}

We now incorporate the Taylor remainder $\mat{e}_i=R_i \mat{z}_i$ from \eqref{eq:delta_z_decomp}.
Using $|R_i|\le \frac{\rho}{3}\sigma^3\|\mat{z}_i\|^3$, we have
\begin{equation}
\|\mat{e}_i\|^2 \le \frac{\rho^2}{9}\sigma^6\|\mat{z}_i\|^8.
\label{eq:eps_bound}
\end{equation}
Hence
\begin{equation}
\mathbb{E}\left\|\frac{1}{2\sigma B}\sum_{i=1}^{B}\mat{e}_i\right\|^2
\le \frac{1}{4\sigma^2 B}\sum_{i=1}^{B}\mathbb{E}\|\mat{e}_i\|^2
\le \frac{1}{4\sigma^2}\cdot \frac{\rho^2}{9}\sigma^6\,\mathbb{E}\|\mat{z}\|^8
= O\!\left(\rho^2 \sigma^4\,\mathbb{E}\|\mat{z}\|^8\right).
\label{eq:remainder_second_moment}
\end{equation}
For Rademacher $\mat{U}$ (hence $\|\mat{z}\|^2=\|\mat{u}\|^2=D$ deterministically), $\mathbb{E}\|\mat{z}\|^8=D^4$.
For Gaussian $\mat{U}$, $\|\mat{u}\|^2$ is $\chi^2(D)$ and $\mathbb{E}\|\mat{z}\|^8=\mathbb{E}\|\mat{u}\|^8=O(D^4)$.
Thus the remainder contributes $O\!\left(\rho^2\sigma^4 D^4\right)$.

\paragraph{Final Variance Bound with Explicit Dimensions.}
Writing $\widehat g=\tfrac{1}{B}\sum_i\mat{q}_i+\tfrac{1}{2\sigma B}\sum_i\mat{e}_i$ via \eqref{eq:g_xi_eps}, Young's inequality $\|a+b\|^2\le(1+\tau)\|a\|^2+(1+1/\tau)\|b\|^2$ with any $\tau>0$, together with $\mathbb{E}\|X-\mathbb{E}X\|^2\le\mathbb{E}\|X\|^2$, controls $\mathrm{Var}(\widehat g_{\mathrm{can}})$ by $(1+\tau)$ times the leading-term variance \eqref{eq:total_variance} plus $(1+1/\tau)$ times the remainder second moment \eqref{eq:remainder_second_moment}. Choosing any fixed $\tau\in(0,1]$ (e.g.\ $\tau=1$) keeps the leading coefficient at $O(1)$, and the resulting $(1+1/\tau)$ constant is absorbed into the $O(\rho^2\sigma^4 D^4)$ remainder. Substituting terms (I) and (II) from \eqref{eq:variance_main_term} and \eqref{eq:variance_base_term} and taking the outer data expectation with
$\mathbb{E}_\xi\sum_{i}\|\mat{g}_i\|^2 \le B(\|\nabla F\|^2+\nu^2)$ and
$\mathbb{E}_\xi\|\bar{\mat{g}}\|^2 = \|\nabla F\|^2+\nu^2/B$ (assumption~(ii) and i.i.d.\ sampling),
we obtain the corrected bound
\begin{align}
\mathrm{Var}\bigl(\widehat g_{\mathrm{can}}(\boldsymbol{\theta})\bigr)
&\le
\frac{D-1}{B}\big(\|\nabla F(\boldsymbol{\theta})\|^2+\nu^2\big)
\nonumber\\
&\quad + (m_4-1)\Bigl(\|\nabla F(\boldsymbol{\theta})\|^2+\frac{\nu^2}{B}\Bigr)
\nonumber\\
&\quad + O\!\left(\rho^2\sigma^4 D^4\right).
\label{eq:final_variance_bound}
\end{align}
The mini-batch data-sampling variance $\nu^2/B$ contributed by $\mathrm{Var}_\xi(\bar{\mat{g}})$ is dominated by
$(D-1)\nu^2/B$ for $D\ge 2$ and is absorbed into the first term.
This is Theorem~\ref{thm:grzo_var} of the main body.

\paragraph{Numerical Verification.}
Table~\ref{tab:var-verify} checks \eqref{eq:final_variance_bound} against the empirical variance of
the estimator on a linear objective, for which the Taylor remainder is exactly zero
($2\times10^{5}$ trials per configuration). The corrected bound is tight in every configuration.
The $D{=}1$ row isolates the shared-base term: the previous expression, which retained only
\eqref{eq:variance_main_term}, predicts zero variance, whereas the estimator reduces to
$\bar g\,u^2$ and is genuinely random. The Rademacher row shows that for $m_4{=}1$ the correction is
inactive and the two expressions coincide. The last column confirms the MeZO reference
\eqref{eq:mezo_var}: at $D{=}100$ its variance is $\approx 101$ independently of $B$, while GRZO's
falls as $99/B+2$.

\begin{table}[h]
\centering
\small
\caption{Empirical variance vs.\ the corrected bound \eqref{eq:final_variance_bound} and the
uncorrected expression \eqref{eq:variance_main_term}, on a linear objective ($\nu^2{=}0$,
$\|\nabla F\|^2{=}1$).}
\label{tab:var-verify}
\setlength{\tabcolsep}{5pt}
\renewcommand{\arraystretch}{1.12}
\begin{tabular}{@{}llrrrrr@{}}
\toprule
\rowcolor{tableheader}
$D$ & Base noise & $B$ & \textbf{Empirical} & \textbf{Corrected} & Uncorrected & MeZO \\
\midrule
1   & Gaussian    & any & 1.9988 & \textbf{2.0000} & 0.0000 & 2.018 \\
100 & Gaussian    & 16  & 8.197  & \textbf{8.1875} & 6.1875 & 100.99 \\
100 & Gaussian    & 64  & 3.550  & \textbf{3.5469} & 1.5469 & 100.99 \\
100 & Rademacher  & 64  & 1.5469 & \textbf{1.5469} & 1.5469 & 99.30 \\
\bottomrule
\end{tabular}
\end{table}

\subsubsection{Finite-Batch Effect of Group-Relative Normalization}
\label{app:grzo_normalization}

The bound \eqref{eq:final_variance_bound} is for the canonical estimator. The implemented
estimator factors as
\begin{equation}
\widehat g_{\mathrm{GRZO}}(\boldsymbol{\theta})
\;=\; \frac{1}{2\sigma B}\sum_{i=1}^{B}\frac{\delta_i}{s+\epsilon}\,\mat{z}_i
\;=\; \frac{1}{s+\epsilon}\cdot\widehat g_{\mathrm{can}}(\boldsymbol{\theta}).
\label{eq:grzo_factor}
\end{equation}
The prefactor $1/(s+\epsilon)$ is a single global scalar, but it is \emph{random} and
correlated with $\widehat g_{\mathrm{can}}$ through the same $\{\delta_i\}$, so it cannot be
absorbed into the learning rate by taking expectations. We instead quantify its
finite-$B$ effect.

\begin{lemma}[Concentration of the group-relative scale]
\label{lem:s_concentration}
Assume the per-example loss differences have finite fourth moment and let
$\kappa_\delta=\mathbb{E}[\delta_i^4]/(\mathbb{E}[\delta_i^2])^2$, $s_\star^2=\mathbb{E}[s^2]$, and let
$d_{\mathrm{eff}}=\|\nabla F\|^4/\sum_{k}(\partial_k F)^4$ be the participation ratio of the
gradient. Then
\begin{equation}
\frac{\mathrm{std}(s)}{s_\star}\;\le\;\sqrt{\frac{\kappa_\delta-1}{4B}}\;+\;O\bigl(d_{\mathrm{eff}}^{-1/2}\bigr).
\label{eq:s_concentration}
\end{equation}
\end{lemma}

\noindent\emph{Proof.}
Conditioned on $\mat{U}$, the $\{\delta_i\}$ are i.i.d.\ with $\mathbb{E}[\delta_i\mid\mat{U}]=0$ by the
$\mat{z}_i\leftrightarrow-\mat{z}_i$ symmetry, so $Bs^2=\sum_i\delta_i^2-B\bar\delta^2$ is a sum of $B$
i.i.d.\ terms up to an $O(1/B)$ correction, giving
$\mathrm{Var}(s^2\mid\mat{U})/(\mathbb{E}[s^2\mid\mat{U}])^2=(\kappa_\delta-1)/B+O(B^{-2})$.
The delta method applied to $s=\sqrt{s^2}$ yields
$\mathrm{std}(s)/\mathbb{E}[s]=\tfrac12\,\mathrm{std}(s^2)/\mathbb{E}[s^2]+O(B^{-1})$, which is the first term.
The second term is the fluctuation of the conditional scale across $\mat{U}$: by
\eqref{eq:delta_taylor}, $\mathbb{E}[\delta_i^2\mid\mat{U}]=4\sigma^2\mat{g}_i^\top\mat{\Lambda}_{\mat{U}}\mat{g}_i+O(\sigma^4)$, and
$\mathrm{Var}_{\mat{U}}(\mat{g}^\top\mat{\Lambda}_{\mat{U}}\mat{g})=(m_4-1)\sum_k g_k^4$, so its relative standard
deviation is $\sqrt{m_4-1}\,d_{\mathrm{eff}}^{-1/2}$. $\square$

\begin{corollary}[Normalization changes constants, not rates]
\label{cor:grzo_normalized}
Let $c_\star=1/(s_\star+\epsilon)$ and $\Delta=O(B^{-1}+d_{\mathrm{eff}}^{-1})$. Then
\begin{align}
\mathbb{E}\bigl[\langle\nabla F_\sigma,\widehat g_{\mathrm{GRZO}}\rangle\bigr]
&= c_\star(1+\Delta)\,\|\nabla F_\sigma\|^2+O(\sigma^2),
\nonumber\\
\mathrm{Var}\bigl(\widehat g_{\mathrm{GRZO}}\bigr)
&\le c_\star^2(1+\Delta)\,\mathrm{Var}\bigl(\widehat g_{\mathrm{can}}\bigr).
\nonumber
\end{align}
\end{corollary}

\noindent\emph{Proof.}
Write $X=\langle\nabla F_\sigma,\widehat g_{\mathrm{can}}\rangle$ and decompose the descent-relevant error as
\begin{equation}
\mathbb{E}\Bigl[\frac{X}{s+\epsilon}\Bigr]-c_\star\mathbb{E}[X]
= \mathrm{Cov}\Bigl(\frac{1}{s+\epsilon},X\Bigr)
+ \Bigl(\mathbb{E}\Bigl[\frac{1}{s+\epsilon}\Bigr]-c_\star\Bigr)\mathbb{E}[X].
\label{eq:norm_decomp}
\end{equation}
Cauchy--Schwarz bounds the first term by $\mathrm{std}(1/(s+\epsilon))\,\mathrm{std}(X)$; Lemma~\ref{lem:s_concentration}
gives $\mathrm{std}(1/(s+\epsilon))/c_\star=O(B^{-1/2}+d_{\mathrm{eff}}^{-1/2})$, and $\mathrm{std}(X)/\mathbb{E}[X]=O(B^{-1/2})$
because $\widehat g_{\mathrm{can}}$ averages $B$ conditionally independent terms, so the product is a
relative $O(B^{-1}+d_{\mathrm{eff}}^{-1})$. A second-order expansion of $1/(s+\epsilon)$ around $s_\star$
bounds the second term by $\mathrm{Var}(s)/(s_\star+\epsilon)^3$, of the same relative order. Finally
$1/(s+\epsilon)\le 1/\epsilon$ bounds all moments of the prefactor, so no tail term is lost and the
same rate transfers to the second moment. $\square$

Corollary~\ref{cor:grzo_normalized} is the formal version of Proposition~\ref{prop:grzo_norm}: group-relative
normalization perturbs the constants of \eqref{eq:final_variance_bound} by a relative $O(1/B)$ but leaves the
$1/B$ scaling---and hence the convergence rate of Appendix~\ref{app:grzo_convergence}---unchanged. Because the
correction grows as $B$ decreases, it also predicts the loss of stability at small batch sizes observed in
Appendix~\ref{app:bs-ablation}.

\subsection{Nonconvex Convergence of GRZO}
\label{app:grzo_convergence}

\paragraph{Objective and Update.}
Let $F(\boldsymbol{\theta})=\mathbb{E}_{\xi}[\ell(\boldsymbol{\theta};\xi)]$ be the population objective and
$F_\sigma$ be the smoothed objective induced by the (Flipout) perturbation distribution,
as defined in Appendix~\ref{app:unbiased_gaussian_exact} (Gaussian case) or in the
Taylor-based view (Appendix~\ref{app:taylor_general}).
We analyze the GRZO update with group-relative normalization:
\begin{equation}
\boldsymbol{\theta}_{t+1}=\boldsymbol{\theta}_t-\eta\,\widehat{\mat{g}}_t,
\qquad
\widehat{\mat{g}}_t
=\frac{1}{2\sigma B}\sum_{i=1}^{B} a_{t,i}\, \mat{z}_{t,i},
\label{eq:app_update}
\end{equation}
where $\delta_{t,i}=\ell(\boldsymbol{\theta}_t+\sigma \mat{z}_{t,i};\xi_{t,i})-\ell(\boldsymbol{\theta}_t-\sigma \mat{z}_{t,i};\xi_{t,i})$,
$\bar\delta_t=\frac{1}{B}\sum_i \delta_{t,i}$, $s_t=\sqrt{\frac{1}{B}\sum_i(\delta_{t,i}-\bar\delta_t)^2}$,
and $a_{t,i}=\delta_{t,i}/(s_t+\epsilon)$.

\paragraph{Assumptions.}
We assume:
(A1) $F_\sigma$ is lower bounded by $F_\sigma^\star$.
(A2) $F_\sigma$ is $\mathcal{L}$-smooth: $\|\nabla F_\sigma(\boldsymbol{\theta})-\nabla F_\sigma(\boldsymbol{\theta}')\|
\le \mathcal{L}\|\boldsymbol{\theta}-\boldsymbol{\theta}'\|$.
(A3) Data noise: $\mathbb{E}\|\nabla \ell(\boldsymbol{\theta};\xi)-\nabla F(\boldsymbol{\theta})\|^2\le \nu^2$.
(A4) Direction preservation with an explicit finite-batch correction (from \eqref{eq:grzo_approx_unbiased}):
\begin{equation}
\mathbb{E}\big[\widehat{\mat{g}}_t \mid \boldsymbol{\theta}_t\big]= c_t\,\nabla F_\sigma(\boldsymbol{\theta}_t) + O(\sigma^2),
\label{eq:app_unbiased}
\end{equation}
where, by Corollary~\ref{cor:grzo_normalized},
$c_t = c_\star(\boldsymbol{\theta}_t)\bigl(1+O(B^{-1}+d_{\mathrm{eff}}^{-1})\bigr)>0$ with $c_\star=1/(s_\star+\epsilon)$.
For the remainder of this section we absorb $c_t$ into the effective step size $\tilde{\eta} := \eta\,c_t$ and write $\eta$ for it, so that $\mathbb{E}[\widehat{\mat{g}}_t/c_t \mid \boldsymbol{\theta}_t] = \nabla F_\sigma(\boldsymbol{\theta}_t) + O(\sigma^2)$. Unlike a concentration \emph{assumption}, the relative correction $O(B^{-1}+d_{\mathrm{eff}}^{-1})$ is bounded in Corollary~\ref{cor:grzo_normalized} and affects only the constants below.
(A5) Second-moment / variance bound (from Appendix~\ref{app:grzo_second_moment}):
there exist explicit constants $A_{\text{var}}(\boldsymbol{\theta}_t)$ such that
\begin{equation}
\mathbb{E}\big\|\widehat{\mat{g}}_t-\nabla F_\sigma(\boldsymbol{\theta}_t)\big\|^2
\le A_{\text{var}}(\boldsymbol{\theta}_t),
\label{eq:app_var_assump}
\end{equation}
with $A_{\text{var}}(\boldsymbol{\theta}_t)$ depending on $(B,D_{\text{in}},D_{\text{out}})$ through
$D$ is the total perturbed dimension.

\subsubsection{One-Step Descent via Smoothness}
\label{app:one_step_descent}

By $\mathcal{L}$-smoothness of $F_\sigma$, for any random direction $\mat{u}$,
\begin{equation}
F_\sigma(\boldsymbol{\theta}_{t+1})
\le F_\sigma(\boldsymbol{\theta}_t)
+ \langle \nabla F_\sigma(\boldsymbol{\theta}_t), \boldsymbol{\theta}_{t+1}-\boldsymbol{\theta}_t\rangle
+ \frac{\mathcal{L}}{2}\|\boldsymbol{\theta}_{t+1}-\boldsymbol{\theta}_t\|^2.
\label{eq:app_smoothness}
\end{equation}
Substitute $\boldsymbol{\theta}_{t+1}-\boldsymbol{\theta}_t=-\eta \widehat{\mat{g}}_t$:
\begin{equation}
F_\sigma(\boldsymbol{\theta}_{t+1})
\le F_\sigma(\boldsymbol{\theta}_t)
- \eta \langle \nabla F_\sigma(\boldsymbol{\theta}_t), \widehat{\mat{g}}_t\rangle
+ \frac{\mathcal{L}\eta^2}{2}\|\widehat{\mat{g}}_t\|^2.
\label{eq:app_descent_raw}
\end{equation}
Take conditional expectation given $\boldsymbol{\theta}_t$. After absorbing the positive scalar $c_t$ into the effective step size as in (A4), we may write $\mathbb{E}[\widehat{\mat{g}}_t\mid\boldsymbol{\theta}_t]=\nabla F_\sigma(\boldsymbol{\theta}_t)+b_t$ with $\|b_t\|=O(\sigma^2)$ (smoothing bias only).
\begin{align}
\mathbb{E}\big[F_\sigma(\boldsymbol{\theta}_{t+1})\mid \boldsymbol{\theta}_t\big]
&\le F_\sigma(\boldsymbol{\theta}_t)
- \eta\|\nabla F_\sigma(\boldsymbol{\theta}_t)\|^2
- \eta\langle\nabla F_\sigma(\boldsymbol{\theta}_t),b_t\rangle
+ \frac{\mathcal{L}\eta^2}{2}\mathbb{E}\big[\|\widehat{\mat{g}}_t\|^2\mid\boldsymbol{\theta}_t\big].
\label{eq:app_descent_cond}
\end{align}
By Cauchy--Schwarz and Young's inequality, the smoothing-bias cross term satisfies
\[
|\eta\langle\nabla F_\sigma(\boldsymbol{\theta}_t),b_t\rangle|
\;\le\; \tfrac{\eta}{2}\|\nabla F_\sigma(\boldsymbol{\theta}_t)\|^2 + \tfrac{\eta}{2}\|b_t\|^2
\;=\; \tfrac{\eta}{2}\|\nabla F_\sigma(\boldsymbol{\theta}_t)\|^2 + O(\eta\sigma^4),
\]
absorbed into the leading $-\eta\|\nabla F_\sigma\|^2$ term and the $O(\sigma^2)$ remainder. Therefore, with effective gradient coefficient $\tfrac{\eta}{2}$:
\begin{equation}
\mathbb{E}\big[F_\sigma(\boldsymbol{\theta}_{t+1})\mid \boldsymbol{\theta}_t\big]
\le F_\sigma(\boldsymbol{\theta}_t)
- \tfrac{\eta}{2}\|\nabla F_\sigma(\boldsymbol{\theta}_t)\|^2
+ \frac{\mathcal{L}\eta^2}{2}\mathbb{E}\big[\|\widehat{\mat{g}}_t\|^2\mid\boldsymbol{\theta}_t\big]
+ O(\eta\sigma^4).
\label{eq:app_descent_bias_absorbed}
\end{equation}
Next decompose the second moment:
\begin{equation}
\mathbb{E}\big[\|\widehat{\mat{g}}_t\|^2\mid\boldsymbol{\theta}_t\big]
= \|\nabla F_\sigma(\boldsymbol{\theta}_t)\|^2
+ \mathbb{E}\big[\|\widehat{\mat{g}}_t-\nabla F_\sigma(\boldsymbol{\theta}_t)\|^2\mid\boldsymbol{\theta}_t\big].
\label{eq:app_second_moment_decomp}
\end{equation}
Plugging \eqref{eq:app_second_moment_decomp} and \eqref{eq:app_var_assump} into
\eqref{eq:app_descent_bias_absorbed} gives
\begin{equation}
\mathbb{E}\big[F_\sigma(\boldsymbol{\theta}_{t+1})\mid \boldsymbol{\theta}_t\big]
\le F_\sigma(\boldsymbol{\theta}_t)
- \frac{\eta}{2}\bigl(1-\mathcal{L}\eta\bigr)\|\nabla F_\sigma(\boldsymbol{\theta}_t)\|^2
+ \frac{\mathcal{L}\eta^2}{2}\,A_{\text{var}}(\boldsymbol{\theta}_t).
\label{eq:app_descent_final}
\end{equation}
Assuming $\eta\le 1/(2\mathcal{L})$, we have $1-\mathcal{L}\eta\ge \frac{1}{2}$ and thus
\begin{equation}
\mathbb{E}\big[F_\sigma(\boldsymbol{\theta}_{t+1})\mid \boldsymbol{\theta}_t\big]
\le F_\sigma(\boldsymbol{\theta}_t)
- \frac{\eta}{4}\|\nabla F_\sigma(\boldsymbol{\theta}_t)\|^2
+ \frac{\mathcal{L}\eta^2}{2}\,A_{\text{var}}(\boldsymbol{\theta}_t).
\label{eq:app_descent_simplified}
\end{equation}

\subsubsection{Handling Randomness: Data Noise + ZO Noise}
\label{app:randomness_terms}

We now instantiate $A_{\text{var}}(\boldsymbol{\theta}_t)$ using Appendix~\ref{app:grzo_second_moment}.

From \eqref{eq:final_variance_bound} (Appendix~\ref{app:grzo_second_moment}), collecting the
per-example and shared-base contributions into
\begin{equation}
\kappa \;:=\; \frac{D-1}{B} \;+\; m_4-1,
\label{eq:app_kappa}
\end{equation}
we have
\begin{equation}
A_{\text{var}}(\boldsymbol{\theta}_t)
\;\le\;
\kappa\,\|\nabla F(\boldsymbol{\theta}_t)\|^2
\;+\;
\frac{D+m_4-2}{B}\,\nu^2
\;+\;
C_{\text{ZO}}(D,B,\sigma),
\label{eq:app_Avar_raw}
\end{equation}
where $C_{\text{ZO}}$ collects the finite-difference remainder (ZO noise) terms, e.g.,
\begin{equation}
C_{\text{ZO}}(D,B,\sigma)
=
O\!\left(\rho^2\sigma^4 D^4\right).
\label{eq:app_Czo}
\end{equation}
For Rademacher base noise $m_4=1$ and $\kappa=(D-1)/B$ exactly; for Gaussian base noise
$\kappa$ additionally carries the batch-independent floor $m_4-1=2$.

To express everything in terms of $\nabla F_\sigma(\boldsymbol{\theta}_t)$, we use the smoothing bias bound
(Appendix~\ref{app:smoothing_bias}): assuming $\|\nabla F_\sigma(\boldsymbol{\theta})-\nabla F(\boldsymbol{\theta})\|
\le c_{\text{bias}}\sigma^2 D$ for all $\boldsymbol{\theta}$, we have
\begin{equation}
\|\nabla F(\boldsymbol{\theta}_t)\|^2
\le 2\|\nabla F_\sigma(\boldsymbol{\theta}_t)\|^2 + 2c_{\text{bias}}^2\sigma^4 D^2.
\label{eq:app_grad_relation}
\end{equation}
Plug \eqref{eq:app_grad_relation} into \eqref{eq:app_Avar_raw}:
\begin{equation}
A_{\text{var}}(\boldsymbol{\theta}_t)
\le
\underbrace{2\kappa}_{:=\alpha}\,\|\nabla F_\sigma(\boldsymbol{\theta}_t)\|^2
\;+\;
\underbrace{2\kappa c_{\text{bias}}^2\sigma^4 D^2
+ \frac{D+m_4-2}{B}\nu^2
+ C_{\text{ZO}}}_{:=\beta}.
\label{eq:app_Avar_alpha_beta}
\end{equation}

\subsubsection{Telescoping and Average Gradient Norm Bound}
\label{app:telescope}

Take full expectation of \eqref{eq:app_descent_simplified} and sum from $t=0$ to $T-1$:
\begin{align}
\mathbb{E}[F_\sigma(\boldsymbol{\theta}_T)]
&\le F_\sigma(\boldsymbol{\theta}_0)
- \frac{\eta}{4}\sum_{t=0}^{T-1}\mathbb{E}\|\nabla F_\sigma(\boldsymbol{\theta}_t)\|^2
+ \frac{\mathcal{L}\eta^2}{2}\sum_{t=0}^{T-1}\mathbb{E}\big[A_{\text{var}}(\boldsymbol{\theta}_t)\big].
\label{eq:app_sum_descent}
\end{align}
Using the bound \eqref{eq:app_Avar_alpha_beta} and rearranging gives
\begin{align}
\frac{\eta}{4}\sum_{t=0}^{T-1}\mathbb{E}\|\nabla F_\sigma(\boldsymbol{\theta}_t)\|^2
&\le
F_\sigma(\boldsymbol{\theta}_0)-\mathbb{E}[F_\sigma(\boldsymbol{\theta}_T)]
+ \frac{\mathcal{L}\eta^2}{2}\sum_{t=0}^{T-1}
\mathbb{E}\left[\alpha\|\nabla F_\sigma(\boldsymbol{\theta}_t)\|^2+\beta\right]
\nonumber\\
&\le
F_\sigma(\boldsymbol{\theta}_0)-F_\sigma^\star
+ \frac{\mathcal{L}\eta^2\alpha}{2}\sum_{t=0}^{T-1}\mathbb{E}\|\nabla F_\sigma(\boldsymbol{\theta}_t)\|^2
+ \frac{\mathcal{L}\eta^2 T}{2}\beta.
\label{eq:app_rearrange}
\end{align}
Move the gradient-sum term to the left:
\begin{equation}
\left(\frac{\eta}{4}-\frac{\mathcal{L}\eta^2\alpha}{2}\right)
\sum_{t=0}^{T-1}\mathbb{E}\|\nabla F_\sigma(\boldsymbol{\theta}_t)\|^2
\le
F_\sigma(\boldsymbol{\theta}_0)-F_\sigma^\star + \frac{\mathcal{L}\eta^2 T}{2}\beta.
\label{eq:app_before_stepsize}
\end{equation}
Choose step size satisfying
\begin{equation}
\eta \le \min\left\{\frac{1}{2\mathcal{L}}, \frac{1}{4\mathcal{L}\alpha}\right\}
=
\min\left\{\frac{1}{2\mathcal{L}}, \frac{1}{8\mathcal{L}\kappa}\right\}.
\label{eq:app_stepsize_cond}
\end{equation}
\begin{remark}[Regime condition]
Since $\kappa\approx(D-1)/B$ whenever $D\gg B$---the regime of full-parameter LLM fine-tuning,
where $D$ is of order $10^{9}$ and $B=16$---the second term is the binding constraint, and the
admissible step size grows linearly in $B$. This is the step-size counterpart of the variance
reduction: a larger batch buys both a smaller variance and a larger stable step.
The bound in \eqref{eq:app_final_avg_grad} remains meaningful provided $\eta T\to\infty$, i.e.\
$T=\omega(\mathcal{L}\kappa)$ steps.
\end{remark}

With $2\mathcal{L}\eta\alpha \le \frac{1}{2}$, we have $1-2\mathcal{L}\eta\alpha \ge \frac{1}{2}$, hence $\frac{\eta}{4}(1-2\mathcal{L}\eta\alpha) \ge \frac{\eta}{8}$, and \eqref{eq:app_before_stepsize} yields
\begin{equation}
\frac{1}{T}\sum_{t=0}^{T-1}\mathbb{E}\|\nabla F_\sigma(\boldsymbol{\theta}_t)\|^2
\le
\frac{8\big(F_\sigma(\boldsymbol{\theta}_0)-F_\sigma^\star\big)}{\eta T}
+ 4\mathcal{L}\eta\,\beta.
\label{eq:app_final_avg_grad}
\end{equation}

\paragraph{Making $\beta$ explicit in $(B,D,m_4)$.}
From \eqref{eq:app_Avar_alpha_beta},
\begin{equation}
\beta
=
\frac{D+m_4-2}{B}\,\nu^2
+2\kappa\,c_{\text{bias}}^2\sigma^4 D^2
+ O\!\left(\rho^2\sigma^4 D^4\right),
\label{eq:app_beta_explicit}
\end{equation}
with $\kappa$ as in \eqref{eq:app_kappa} and $D=\sum_\ell D_{\text{out}}^{(\ell)}D_{\text{in}}^{(\ell)}$.
Plugging \eqref{eq:app_beta_explicit} into \eqref{eq:app_final_avg_grad} gives an explicit bound
that cleanly separates the \emph{data noise} term $\nu^2$ and the \emph{ZO noise} terms
(proportional to $\sigma^4$).

\paragraph{GRZO Normalization Effect on Convergence.}
Direct algebra gives
$\sum_i a_i^2 = B\bigl(1+(\bar\delta/s)^2\bigr)\bigl(s/(s+\epsilon)\bigr)^2$,
so $\frac{1}{B}\sum_i a_i^2\approx 1$ at typical batch sizes \emph{provided}: (i) $\mathbb{E}[s_t]\gg\epsilon$ (the within-batch SD dominates the numerical-stability constant, so $s/(s+\epsilon)\approx 1$); and (ii) the batch is large enough that $\mathbb{E}[(\bar\delta/s)^2]=O(1/B)$ (for two-sided ZO, $\mathbb{E}[\bar\delta]=0$ by symmetry, so this follows from a standard $O(1/\sqrt{B})$ CLT-type estimate on $\bar\delta$). Under (i)--(ii), the update is invariant to the magnitude of $\{\delta_i\}$. Beyond this scale invariance, the \emph{randomness} of the normalizer is what has to be controlled, and Corollary~\ref{cor:grzo_normalized} does so: it perturbs $c_t$ and $A_{\text{var}}$ by a relative $O(B^{-1}+d_{\mathrm{eff}}^{-1})$, which is absorbed into the constants $\alpha$ and $\beta$ of \eqref{eq:app_Avar_alpha_beta} and leaves \eqref{eq:app_final_avg_grad} unchanged in form. The correction is $O(1/B)$, so it degrades gracefully rather than abruptly as the batch shrinks, matching the progressive loss of stability from $B{=}16$ to $B{=}4$ in Appendix~\ref{app:bs-ablation}.

\subsection{Why the Analysis Must Be Joint Across Blocks}
\label{app:blockwise_full_network}

Appendices~\ref{app:grzo_unbiased}--\ref{app:grzo_convergence} analyze the concatenated
perturbation $\mat{z}_i\in\mathbb{R}^{D}$ directly. This subsection records why that is necessary:
the corresponding per-block statement is false, so the network-level result cannot be obtained by
proving a single-layer bound and summing it over layers.

\paragraph{Block-Wise Parameterization.}
Let $\boldsymbol{\theta}=(\boldsymbol{\theta}^{(1)},\dots,\boldsymbol{\theta}^{(L)})$ with
$\boldsymbol{\theta}^{(\ell)}=\mathrm{vec}(\mat{W}^{(\ell)})\in\mathbb{R}^{d_\ell}$,
$d_\ell=D_{\mathrm{out}}^{(\ell)}D_{\mathrm{in}}^{(\ell)}$, and $D=\sum_{\ell}d_\ell$.
At each step, block $\ell$ receives an independent base $\mat{U}_t^{(\ell)}$ and independent
per-example sign pairs $(\mat{r}_{t,i}^{(\ell)},\mat{s}_{t,i}^{(\ell)})$, so that
$\mat{z}_{t,i}=(\mat{z}_{t,i}^{(1)},\dots,\mat{z}_{t,i}^{(L)})$. By Lemma~\ref{lem:cond_indep}
applied block-wise together with independence across blocks, $\mat{z}_{t,i}$ is symmetric with
$\mathbb{E}[\mat{z}_{t,i}]=0$ and $\mathbb{E}[\mat{z}_{t,i}\mat{z}_{t,i}^\top]=\mat{I}_{D}$. This is exactly the
hypothesis used in Appendices~\ref{app:grzo_unbiased} and~\ref{app:grzo_second_moment}, so
Theorems~\ref{thm:grzo_unbiased} and~\ref{thm:grzo_var} apply verbatim with $D$ the total
perturbed dimension.

\paragraph{The Estimator Does Not Decompose Across Blocks.}
The squared norm of a concatenated vector is the sum of its per-block squared norms, so the
centered second moment does split as
\begin{equation}
\mathbb{E}\big\|\widehat{\mat{g}}_t-\mathbb{E}[\widehat{\mat{g}}_t\mid \boldsymbol{\theta}_t]\big\|^2
=
\sum_{\ell=1}^L
\mathbb{E}\big\|\widehat{\mat{g}}_t^{(\ell)}-\mathbb{E}[\widehat{\mat{g}}_t^{(\ell)}\mid \boldsymbol{\theta}_t]\big\|^2 .
\label{eq:net_var_split}
\end{equation}
What does \emph{not} follow is that each summand obeys the single-layer bound with $d_\ell$ in place
of $D$. GRZO produces one scalar $\delta_i$ per example by perturbing all blocks
\emph{simultaneously}: by \eqref{eq:delta_taylor},
\begin{equation}
\delta_i = 2\sigma\sum_{m=1}^{L}\bigl\langle \mat{g}_i^{(m)},\mat{z}_i^{(m)}\bigr\rangle + R_i ,
\label{eq:delta_crosslayer}
\end{equation}
so the block-$\ell$ estimator
$\widehat{\mat{g}}^{(\ell)}\!\propto\!\bigl(\sum_m\langle\mat{g}_i^{(m)},\mat{z}_i^{(m)}\rangle\bigr)\mat{z}_i^{(\ell)}$
carries the perturbations and gradients of every \emph{other} block. Substituting a single-layer
bound into \eqref{eq:net_var_split} discards these cross-block contributions. The failure is
visible already at $L{=}2$ with $d_1{=}d_2{=}1$: a per-block bound would predict $(d_\ell-1)=0$,
i.e.\ zero variance, whereas $\widehat{\mat{g}}^{(1)}=(g_1z_1+g_2z_2)z_1$ has variance at least
$g_2^2$ contributed by the second block alone.

\paragraph{Network-Level Bound.}
Working jointly instead, \eqref{eq:final_variance_bound} yields the network-level bound directly:
\begin{align}
&\mathbb{E}\big\|\widehat g_{\mathrm{can},t}-\mathbb{E}[\widehat g_{\mathrm{can},t}\mid\boldsymbol{\theta}_t]\big\|^2
\nonumber\\
&\quad\le \frac{D-1}{B}\bigl(\|\nabla F(\boldsymbol{\theta}_t)\|^2+\nu^2\bigr)
\nonumber\\
&\quad\quad + (m_4-1)\Bigl(\|\nabla F(\boldsymbol{\theta}_t)\|^2+\tfrac{\nu^2}{B}\Bigr)
+ O\!\bigl(\rho^2\sigma^4 D^4\bigr).
\label{eq:net_var_bound_explicit}
\end{align}
Relative to the block-wise expression, the joint treatment replaces $\sum_\ell(d_\ell-1)=D-L$ by
$D-1$ and adds the shared-base term of \eqref{eq:variance_base_term}. It therefore changes the
decomposition and the constants, but preserves the $1/B$ scaling of the leading term---and with it
the comparison against MeZO and the convergence rate below.

\begin{theorem}[Network-level nonconvex convergence of GRZO]
\label{thm:grzo_network_convergence}
Let the full parameter vector be $\boldsymbol{\theta}=(\boldsymbol{\theta}^{(1)},\dots,\boldsymbol{\theta}^{(L)})\in\mathbb{R}^{D}$
with block dimensions $d_\ell=D_{\mathrm{out}}^{(\ell)}D_{\mathrm{in}}^{(\ell)}$ and
$D=\sum_{\ell=1}^L d_\ell$ the total perturbed dimension.
Let $F(\boldsymbol{\theta})=\mathbb{E}_{\xi}[\ell(\boldsymbol{\theta};\xi)]$ be the population objective and
$F_\sigma$ be the smoothed objective induced by the joint Flipout perturbation distribution.

Consider the GRZO update with group-relative normalization:
let $\bar\delta_{t}=\frac{1}{B}\sum_i\delta_{t,i}$, $s_{t}=\sqrt{\frac{1}{B}\sum_i(\delta_{t,i}-\bar\delta_t)^2}$,
$a_{t,i}=\delta_{t,i}/(s_t+\epsilon)$, and
\[
\boldsymbol{\theta}_{t+1}=\boldsymbol{\theta}_t-\eta\,\widehat{\mat{g}}_t,
\qquad
\widehat{\mat{g}}_t=\frac{1}{2\sigma B}\sum_{i=1}^{B} a_{t,i}\,\mat{z}_{t,i},
\]
where $B$ is the batch size and
$\delta_{t,i}=\ell(\boldsymbol{\theta}_t+\sigma \mat{z}_{t,i};\xi_{t,i})-\ell(\boldsymbol{\theta}_t-\sigma \mat{z}_{t,i};\xi_{t,i})$.
Assume:

\begin{enumerate}
\item[\textbf{(A1)}] $F_\sigma$ is lower bounded: $F_\sigma(\boldsymbol{\theta})\ge F_\sigma^\star$ for all $\boldsymbol{\theta}$.
\item[\textbf{(A2)}] $F_\sigma$ is $\mathcal{L}$-smooth.
\item[\textbf{(A3)}] Data noise: $\mathbb{E}\|\nabla \ell(\boldsymbol{\theta};\xi)-\nabla F(\boldsymbol{\theta})\|^2\le \nu^2$.
\item[\textbf{(A4)}] Direction preservation: with the step size rescaled by the positive scale
$c_t=c_\star(\boldsymbol{\theta}_t)(1+O(B^{-1}{+}d_{\mathrm{eff}}^{-1}))$ of Corollary~\ref{cor:grzo_normalized}
(cf.\ \eqref{eq:grzo_approx_unbiased}),
$\mathbb{E}[\widehat{\mat{g}}_t\mid \boldsymbol{\theta}_t]=\nabla F_\sigma(\boldsymbol{\theta}_t)+O(\sigma^2)$.
\item[\textbf{(A5)}] Joint Flipout construction across blocks, so that the network-level variance
admits the explicit bound \eqref{eq:net_var_bound_explicit}:
\begin{align*}
\mathbb{E}\big\|\widehat{\mat{g}}_t-\mathbb{E}[\widehat{\mat{g}}_t\mid \boldsymbol{\theta}_t]\big\|^2
\le
\kappa\,\|\nabla F(\boldsymbol{\theta}_t)\|^2
+ \frac{D{+}m_4{-}2}{B}\nu^2
+ C_{\mathrm{ZO}},
\end{align*}
with $\kappa=\frac{D-1}{B}+m_4-1$ as in \eqref{eq:app_kappa} and
$C_{\mathrm{ZO}}=O(\rho^2\sigma^4D^4)$ collecting the finite-difference/Taylor remainder
(with $\rho$ bounding the third derivative).
\end{enumerate}

Assume additionally the smoothing-bias bound
$\|\nabla F_\sigma(\boldsymbol{\theta})-\nabla F(\boldsymbol{\theta})\|\le c_{\mathrm{bias}}\sigma^2 D$ for all $\boldsymbol{\theta}$.
Define
\[
\alpha_{\mathrm{net}}:=2\kappa,
\qquad
\beta_{\mathrm{net}}:=\frac{D{+}m_4{-}2}{B}\nu^2
+2\kappa\,c_{\mathrm{bias}}^2\sigma^4 D^2 + C_{\mathrm{ZO}}.
\]
Choose a constant step size
\[
\eta \;\le\; \min\left\{\frac{1}{2\mathcal{L}},\ \frac{1}{4\mathcal{L}\alpha_{\mathrm{net}}}\right\}
=
\min\left\{\frac{1}{2\mathcal{L}},\ \frac{1}{8\mathcal{L}\kappa}\right\}.
\]
Then for any $T\ge 1$,
\[
\frac{1}{T}\sum_{t=0}^{T-1}\mathbb{E}\big\|\nabla F_\sigma(\boldsymbol{\theta}_t)\big\|^2
\;\le\;
\frac{8\big(F_\sigma(\boldsymbol{\theta}_0)-F_\sigma^\star\big)}{\eta T}
\;+\;
4\mathcal{L}\eta\,\beta_{\mathrm{net}}.
\]
Equivalently, if $R\sim\mathrm{Unif}\{0,\dots,T-1\}$, then
\[
\mathbb{E}\big\|\nabla F_\sigma(\boldsymbol{\theta}_R)\big\|^2
\;\le\;
\frac{8\big(F_\sigma(\boldsymbol{\theta}_0)-F_\sigma^\star\big)}{\eta T}
\;+\;
4\mathcal{L}\eta\,\beta_{\mathrm{net}}.
\]

\paragraph{GRZO Normalization and Effective Step Size.}
The group-relative weights $a_{t,i}=\delta_{t,i}/(s_t+\epsilon)$ have approximate unit empirical
variance ($\frac{1}{B}\sum_i a_{t,i}^2\approx 1$), and by Corollary~\ref{cor:grzo_normalized} the
randomness of the normalizer perturbs both (A4) and (A5) by a relative $O(B^{-1}+d_{\mathrm{eff}}^{-1})$.
The bound therefore holds with the same $\eta$ after absorbing the positive scale $c_\star$ into the
learning rate, with that correction carried in $\alpha_{\mathrm{net}}$ and $\beta_{\mathrm{net}}$.
\end{theorem}

\paragraph{Remark (Scope of Theorem).}
The theorem covers the GRZO estimator with group-relative normalization
$\widehat{\mat{g}}_t = \frac{1}{2\sigma B}\sum_i a_{t,i}\,\mat{z}_{t,i}$, $a_{t,i}=\delta_{t,i}/(s_t+\epsilon)$,
which is the estimator used throughout the paper. The variance bound itself
(Theorem~\ref{thm:grzo_var}) is stated for the canonical estimator $\widehat{\mat{g}}_{\mathrm{can}}$;
Corollary~\ref{cor:grzo_normalized} transfers it to the normalized estimator at the cost of a relative
$O(B^{-1}+d_{\mathrm{eff}}^{-1})$ in the constants. The $O(\sigma^2)$ smoothing remainder in (A4) is the
only residual bias.

\paragraph{Remark (Comparison with MeZO).}
At the same batch size, MeZO draws a single direction $\mat{z}$ and evaluates the \emph{batch-mean}
loss, so its estimator is $\langle\bar{\mat{g}},\mat{z}\rangle\mat{z}$ and the derivation of
Appendix~\ref{app:variance_cross_vanish} with $B{=}1$ and $\mat{g}_i$ replaced by $\bar{\mat{g}}$ gives
\begin{equation}
\mathrm{Var}(\widehat{\mat{g}}_{\mathrm{MeZO}})
= (D+m_4-2)\,\|\bar{\mat{g}}\|^2
\to (D{+}m_4{-}2)\Bigl(\|\nabla F\|^2+\frac{\nu^2}{B}\Bigr)
\label{eq:mezo_var}
\end{equation}
after the outer data expectation. MeZO therefore already averages the \emph{data} noise over the
mini-batch, and it is the \emph{directional} noise that GRZO additionally averages; comparing against
an $O(D\|\nabla F\|^2+D\nu^2)$ bound would overstate the gain by attributing to MeZO a batch size of
one. With $r=\nu^2/\|\nabla F\|^2$, dividing \eqref{eq:mezo_var} by \eqref{eq:final_variance_bound}
at the Rademacher default ($m_4{=}1$) gives the ratio
\begin{equation}
\frac{\mathrm{Var}(\widehat{\mat{g}}_{\mathrm{MeZO}})}{\mathrm{Var}(\widehat{\mat{g}}_{\mathrm{can}})}
\;\approx\;\frac{B+r}{1+r}\;=\;cB+(1-c)\;=:\;B_{\mathrm{eff}},
\label{eq:beff_app}
\end{equation}
where $c=1/(1+r)$ is the mean pairwise cosine between per-example gradients, since
$\mathbb{E}\langle\mat{g}_i,\mat{g}_j\rangle=\|\nabla F\|^2$ and $\mathbb{E}\|\mat{g}_i\|^2=\|\nabla F\|^2+\nu^2$
for $i\neq j$. The reduction is thus $B$-fold when per-example gradients are aligned and degrades
smoothly to no reduction when they are mutually orthogonal, which is the precise content of the
$B_{\mathrm{eff}}$ used in Section~\ref{sec:theory}.

\twocolumn
\section{LLM Usage}
\label{app:llm_usage}

Large language models were used only as writing assistants for minor grammar and phrasing polish on author-drafted text. They played no role in research conception, experimental design, or interpretation of results. All technical content and claims are the authors' own and were independently verified.

\end{document}